\documentclass{article}

\usepackage[preprint]{preprint}

\usepackage{amsmath}
\usepackage{amssymb}
\usepackage{mathtools}
\usepackage{amsthm}
\usepackage{graphicx}
\usepackage{booktabs}
\usepackage{multirow}
\usepackage{pifont}
\usepackage{algorithm}
\usepackage{algorithmic}
\usepackage[hyphens]{url}
\theoremstyle{plain}
\newtheorem{theorem}{Theorem}
\newtheorem{proposition}[theorem]{Proposition}
\newtheorem{lemma}[theorem]{Lemma}
\theoremstyle{definition}

\newtheorem{remark}[theorem]{Remark}

\title{When Can Conformal Risk Control Certify LLM Outputs?\\
Bounds, Impossibility, and Adaptation for Structured Generation}

\author{Varun Kotte\\ Independent Researcher\\ \texttt{kottevarun@gmail.com}}

\begin{document}
\maketitle

\begin{abstract}
Large language models (LLMs) deployed for structured generation (NER, JSON extraction, QA, classification) lack formal reliability guarantees, and heuristic abstention policies violate user-specified risk targets on 7.5--12.5\% of settings, with no per-deployment guarantee. We characterize \emph{when} conformal risk control (CRC), which tunes an abstention threshold to certify a target risk, can certify structured LLM outputs and when it provably cannot: a synthesis of known distribution-free tools with new sharpenings and a constructive dual, not a new conformal primitive. First, a sharpened, attained feasibility frontier: when base risk $\mu$ exceeds target $\alpha$, any distribution-free method must abstain on at least $(\mu-\alpha)/(M-\alpha)$ of inputs, $M = \operatorname{ess\,sup} R$ (on our $F_1$/exact-match losses $M{=}1$, matching the standard floor; strictly tighter only when an audited $M<1$ holds), and achievable abstention decomposes into this floor plus a dispersion gap plus a score-ordering gap. Second, a proven certification phase diagram over 716 configurations (six open-weight models 3B--72B, eight datasets, six scores): at strict targets ($\alpha \leq 0.20$) the certified sets are nested (51/72/80 certified; Bernstein the largest upgrade over Hoeffding, +41\%; e-CRC best under scarcity), while at relaxed targets the Hoeffding--Bernstein ordering reverses at a closed-form frontier $\alpha^*(n,\delta)$, with minimal sufficient set $\{\text{Hoeffding}, \text{e-CRC}\}$. Third, a negative shift result with a constructive dual: under cross-dataset shift ($\mu > \alpha$) the target is violated on 14/16 transfers by static CRC and every tested ACI step size, yet a full-feedback anytime-valid monitor certifies 0/16 while remaining non-vacuous. The result is a three-step deployment recipe: check feasibility, select bound and score, then re-check under shift.
\end{abstract}

\section{Introduction}
\label{sec:intro}

When an LLM extracts the wrong medication from a clinical note, misses a liability clause in a contract, or pulls the wrong account number from a financial filing, the error is silent: the model reports full confidence and nothing downstream flags it. LLMs are increasingly deployed for such structured generation (extracting named entities, parsing JSON, answering factual questions, classifying inputs), where the output must match a schema and can be scored against ground truth, yet they give no guarantee on output quality.

Standard uncertainty quantification (token probabilities, entropy thresholds, temperature scaling) is miscalibrated for structured outputs where the loss is defined at the level of entities, fields, or semantic units \citep{guo2017calibration,kadavath2022know,kuhn2023semantic,farquhar2024semanticentropy}. Heuristic abstention policies violate user-specified risk targets on 7.5--12.5\% of settings, with no guarantee for any individual deployment (Appendix Table~\ref{tab:baselines}). Concretely: to guarantee $\leq 10\%$ entity-level error, if the model's base risk $\mu$ on the deployment population exceeds $10\%$, no post-hoc method can hit the target without heavy abstention. A heuristic never sees this floor; our impossibility result computes it in closed form.

We ask a prior question, before ``which bound is tightest?'': \emph{when} can conformal risk control (CRC)~\citep{angelopoulos2024crc,bates2021distribution} certify structured LLM outputs at all, and when is it provably impossible? The algebra behind our floor is elementary; the contribution is recognizing it as a deployment-facing, tight feasibility frontier: when $\mu > \alpha$, no distribution-free method certifies without abstaining on $\geq (\mu-\alpha)/(M-\alpha)$ of examples, so the first question is ``is $\mu < \alpha$?''. We then establish a certification hierarchy (Hoeffding $\subseteq$ Bernstein $\subseteq$ e-CRC; finite-sample under a variance condition, asymptotic for the second) with a closed-form reversal frontier and a Pareto-necessity result at relaxed targets, and evaluate ACI~\citep{gibbs2021aci} under cross-dataset transfer alongside a constructive anytime-valid counterpart. Concurrent CRC-style work for LLMs asks \emph{which} conformal procedure to apply~\citep{quach2024conformallm,mohri2024conformalfactuality,abbasiyadkori2024conformalabstention,gui2024conformalalignment,kotte2026pasc}; we study \emph{when} any such procedure can succeed under task-specific structured losses (see Related Work).

\paragraph{Contributions.} Our contribution is a characterization---a synthesis of known distribution-free tools with new sharpenings and a constructive dual, not a new conformal primitive.
\begin{enumerate}
\item \textbf{Sharpened two-sided feasibility frontier.} We upgrade the abstention floor to the \emph{attained} $(\mu-\alpha)/(M-\alpha)$, $M = \operatorname{ess\,sup} R$ (Proposition~\ref{prop:impossibility}; strictly tighter than the $M{=}1$ floor when an audited $M<1$ holds), with an exact extremal-law reachability condition, and decompose achievable abstention into this floor plus a marginal dispersion gap plus a score ordering gap that a better-ranked score can shrink (Remark~\ref{rem:decomp}), tying the impossibility result to score quality. This specializes---rather than reinvents---known-distribution reject-option optima~\citep{chow1970optimum,elyaniv2010selective,franc2023reject}.
\item \textbf{A proven certification phase diagram.} We sharpen a known pairwise bound-width comparison~\citep{maurer2009empirical} into the hierarchy $\Lambda^*_{\mathrm{Hoeff}} \subseteq \Lambda^*_{\mathrm{Bern}} \subseteq \Lambda^*_{\mathrm{e\text{-}CRC}}$ (Proposition~\ref{prop:ordering}) and, at relaxed targets, a closed-form reversal frontier $\alpha^*(n,\delta)$ (exact on binary losses) with a Pareto-necessity result: no single bound dominates, e-CRC subsumes Bernstein, so the minimal sufficient set is $\{\text{Hoeffding}, \text{e-CRC}\}$ (Proposition~\ref{prop:phase}). Verified across 716 configurations: 51/72/80 certified at strict targets (Hoeffding$\to$Bernstein +41\%; e-CRC 4.2\% vs.\ 1.9\% at 10\% calibration; relaxing $\alpha$ unlocks 28\% NER, 13\% QA, 19\% CLS at $\alpha=0.40$).
\item \textbf{A negative shift result with a constructive counterpart.} Under cross-dataset shift with $\mu > \alpha$, the target is violated on 14 of 16 transfers by static CRC and every tested ACI step size ($\gamma \in [0.001, 0.1]$, 7 values); a full-feedback anytime-valid monitor certifies 0/16 on the same transfers while emitting 65--97\% on feasible streams (\S\ref{sec:monitor}), and an emit-only obstruction (Proposition~\ref{prop:obstruction}) shows ACI's failure is a feedback-model property, not tuning. (ACI still helps under gradual shift: 60\%$\to$4\% temporal, $<$12\% severity-sweep.)
\item \textbf{An open benchmark, all configurations reported.} Evaluation spans six models (3B--72B) from three families, eight datasets, four tasks, and six scores; all 716 configurations are reported with no post-hoc exclusion, including the negative shift result above. Calibrated score fusion is analyzed as a score-quality question (\S\ref{sec:fusion}): on a disjoint split it yields only marginal AUROC gains and does not expand the certified set.
\end{enumerate}

No prior line of work combines all five ingredients---applicability to structured LLM outputs, risk-based (not merely coverage-based) control, the abstention lower bound, the bound hierarchy, and shift adaptation (Appendix Table~\ref{tab:positioning_supp}).

\section{Method}
\label{sec:method}

\paragraph{Problem setup.} An LLM $f: \mathcal{X} \to \mathcal{Y}$ produces $\hat{y} = f(x)$; a score $s(x, \hat{y}) \in [0,1]$ quantifies reliability, and a selective predictor emits when $s \geq \lambda$. We seek $\lambda^*$ with $\mathbb{E}[R_{\mathrm{task}}(\hat{Y}, Y) \mid s(X) \geq \lambda^*] \leq \alpha$ with probability $\geq 1 - \delta$, where $R_{\mathrm{task}}$ is a task-specific risk and $\alpha$ the target. The four risks $R_{\mathrm{NER}} = 1 - F_1^{\mathrm{entity}}$, $R_{\mathrm{JSON}} = 1 - F_1^{\mathrm{field}}$, $R_{\mathrm{QA}} = 1 - \mathbb{I}(\mathrm{EM})$, $R_{\mathrm{CLS}} = 1 - \mathbb{I}(\mathrm{correct})$ are bounded in $[0,1]$, satisfying CRC requirements~\citep{angelopoulos2024crc,angelopoulos2024ltt} (background in Appendix~\ref{app:background}).

\paragraph{CRC calibration.} Given calibration $\{(x_i, y_i)\}_{i=1}^n$, we compute scores $s_i$ and risks $r_i$. Thresholds come from a fixed grid $\Lambda$ of 200 evenly spaced points on $[0,1]$; the emit set is $E_\lambda = \{i : s_i \geq \lambda\}$, $n_\lambda = |E_\lambda|$, and candidates are grid points with $n_\lambda \geq 20$ (a pre-registered, score-deterministic rule). All three bounds scan candidates in \emph{decreasing} $\lambda$, test each at level $\delta$ (for standard CRC, $U_H(\lambda) = \hat{R}(\lambda) + \sqrt{\log(2/\delta)/(2 n_\lambda)} \leq \alpha$), stop at the first failure, and set $\lambda^*$ to the smallest certified threshold, controlling family-wise error without multiplicity correction (Appendix Lemma~\ref{lem:gridsearch}). At test time we emit when $s \geq \lambda^*$, with the guarantee $\mathbb{P}(\mathbb{E}[R \mid s \geq \lambda^*] > \alpha) \leq \delta$ for i.i.d.\ calibration data.

\subsection{Betting-Based CRC (e-CRC)}

The Hoeffding UCB ignores variance. The empirical Bernstein bound~\citep{maurer2009empirical} incorporates sample variance $\hat{\sigma}^2$:
\begin{equation}
U_B(\lambda) = \hat{R} + \sqrt{\frac{2 \hat{\sigma}^2 \log(2/\delta)}{n_\lambda}} + \frac{7 \log(2/\delta)}{3(n_\lambda - 1)},
\end{equation}
where $\hat{\sigma}^2$ is the empirical variance with the $1/n$ convention, matching our implementation (Appendix~\ref{app:proofs}). This is tighter than Hoeffding whenever the emit-set variance satisfies an explicit finite-sample condition $\hat{\sigma}^2 \leq c(n_\lambda, \delta)$ whose large-$n$ limit is the classical $\hat{\sigma}^2 < 1/4$ (Proposition~\ref{prop:ordering}); the asymptotic condition holds for all our datasets (Table~\ref{tab:base_risk}). We go further using the testing-by-betting framework~\citep{shafer2021testing,waudbysmith2024betting,ramdas2023safe,grunwald2024safetesting,vovk2021evalues}. For each candidate $\lambda$, we construct a wealth process over the emit set $E_\lambda$: $W_0 = 1$, $W_j = W_{j-1}(1 + \kappa_j(\alpha - r_j))$, where $\kappa_j$ is a Kelly-style bet clipped to $[0, 0.5]$ with $\kappa_1 = 0$ (no bet before the first observation). The e-value is the minimum wealth $W_m$ over 22 pre-registered orderings of the emit set (score-descending, score-ascending, and 20 seeded random permutations); if this minimum reaches $1/\delta$, the test at $\lambda$ passes. Threshold selection uses the same descending fixed-sequence scan as above (Appendix~\ref{app:algos} for full pseudocode).

\begin{theorem}[Betting-Based Risk Validity]
\label{thm:ecrc_valid}
For i.i.d.\ calibration data and any fixed threshold $\lambda$, the e-CRC test satisfies $\mathbb{P}(\mathbb{E}[R(\lambda)] > \alpha \text{ and } \lambda \text{ certified}) \leq \delta$.
\end{theorem}
\begin{proof}[Proof sketch]
Under $H_0: \mathbb{E}[R] \geq \alpha$, non-negative bets make the wealth $\{W_j\}$ a non-negative supermartingale. By Ville's inequality~\citep{howard2021time}, $\mathbb{P}(\sup_j W_j \geq 1/\delta) \leq \delta$. Full proof in Appendix~\ref{app:proofs}.
\end{proof}

Theorem~\ref{thm:ecrc_valid} is stated for a fixed $\lambda$; Appendix~\ref{app:proofs} (Lemma~\ref{lem:gridsearch}) extends validity to the scan-selected $\hat{\lambda}$ via fixed-sequence testing with early stopping, which requires no risk-monotonicity assumption.

\begin{proposition}[Bound Ordering]
\label{prop:ordering}
Fix a calibration set with bounded risks $r_i \in [0,1]$, a shared candidate threshold set, and confidence level $\delta$; write $L = \log(2/\delta)$, and let $n_\lambda \geq 2$ and $\hat{\sigma}^2_\lambda$ denote the size and empirical variance of the emit set $E_\lambda$.
\emph{(i) Hoeffding $\subseteq$ Bernstein (finite-sample).} For an emit set of size $n \geq 2$, $U_B(\lambda) \leq U_H(\lambda)$ holds if and only if
\[
\hat{\sigma}^2 \leq c(n,\delta) := \frac{n}{2L}\left(\sqrt{\frac{L}{2n}} - \frac{7L}{3(n-1)}\right)^{\!2},
\]
valid when $\sqrt{L/(2n)} > 7L/(3(n-1))$ (which holds for all $n$ above a small threshold depending only on $\delta$; e.g.\ $n \geq 35$ at $\delta = 0.1$); when this positivity condition fails, $U_B > U_H$ for every $\hat{\sigma}^2 \geq 0$. Consequently, $\Lambda^*_{\mathrm{Hoeff}} \subseteq \Lambda^*_{\mathrm{Bern}}$ holds deterministically on any calibration set on which every emit set satisfies $\hat{\sigma}^2_\lambda \leq c(n_\lambda, \delta)$; as $n \to \infty$, $c(n,\delta) \to 1/4$, recovering the classical condition $\hat{\sigma}^2 < 1/4$. In our experiments this condition held on 99.6\% of the emit sets at which the Hoeffding UCB certifies (those on which the inclusion depends) and on 50 of the 51 selected emit sets of Hoeffding-certified configurations, but on only 40.7\% of all candidate emit sets; at strict targets ($\alpha \leq 0.20$) the certified sets were nested without exception, while at relaxed targets the ordering reverses exactly where the condition fails (\S\ref{sec:relax}).
\emph{(ii) Bernstein $\subseteq$ e-CRC (asymptotic and empirical).} Under any alternative with $\mathbb{E}[R] < \alpha$, the clipped Kelly-style bet has strictly positive expected log-wealth drift, so the wealth test is consistent: as $n \to \infty$ it certifies every fixed threshold whose true risk is below $\alpha$, hence (asymptotically) every threshold Bernstein certifies. We make no deterministic finite-sample claim for this inclusion. Empirically it held without exception: no configuration was certified by Bernstein but not by e-CRC at any $\alpha$ (0.05--0.20 across five seeds, 3{,}580 paired evaluations, and $\alpha \in \{0.25, 0.30, 0.40\}$).
\end{proposition}

Intuition: Hoeffding treats $R$ as worst-case on $[0,1]$; Bernstein tightens when $\hat{\sigma}^2 \leq c(n,\delta)$; e-CRC adapts sequentially via betting. The tightening is regime-dependent (Bernstein under low variance, e-CRC under scarcity), and at relaxed targets the Hoeffding--Bernstein ordering reverses (\S\ref{sec:relax}). Proof in Appendix~\ref{app:proofs}.

\subsection{Fundamental Limits of Certification}

\begin{proposition}[Sharpened two-sided abstention floor]
\label{prop:impossibility}
Let $\mu = \mathbb{E}[R]$ be the base risk and $M = \operatorname*{ess\,sup} R \leq 1$, with $\alpha < M$. Any selective predictor whose emitted risk satisfies $\mathbb{E}[R \mid \mathrm{emit}] \leq \alpha$ must abstain with probability
\[
\mathbb{P}(\mathrm{abstain}) \;\geq\; \frac{\mu - \alpha}{M - \alpha} \;\geq\; \frac{\mu - \alpha}{1 - \alpha}.
\]
The bound is \emph{attained}: the two-point law $R \in \{0, M\}$ paired with a perfectly ordering score meets it with equality, and a general law reaches it iff $\mathbb{P}(R = M) \geq (\mu - \alpha)/(M - \alpha)$; it is thus a tight feasibility frontier, not merely a lower bound. For losses bounded away from $1$, i.e.\ when an audited $M<1$ holds, $(\mu-\alpha)/(M-\alpha)$ strictly sharpens the $M=1$ floor and is deployable by auditing $M$ on held-out data. On our $F_1$/exact-match losses a complete per-example miss gives $R=1$, so the audited $M=1.00$ for every NER/QA/CLS dataset and the sharpening coincides with the standard floor there; only JSON field-$F_1$ has $M\leq 0.33$, and there $\mu<\alpha$ already (audited values in Appendix~\ref{app:round2}). The population statement is exact; for a known or audited $M$ and $\hat\mu$ from $n$ held-out points at level $\delta/2$, the plug-in correction is $\varepsilon/(M-\alpha)$, $\varepsilon = \sqrt{\log(2/\delta)/(2n)}$ (Appendix~\ref{app:proofs}). Tables in this paper report the conservative $M=1$ instance $(\mu-\alpha)/(1-\alpha)$, which remains a valid floor.
\end{proposition}

\begin{proof}[Proof sketch]
Decompose $\mu = \mathbb{E}[R\phi] + \mathbb{E}[R(1-\phi)]$ for emit indicator $\phi$; validity gives $\mathbb{E}[R\phi] \leq \alpha\,\mathbb{E}[\phi]$ and $R \leq M$ pointwise gives $\mathbb{E}[R(1-\phi)] \leq M(1-p)$, $p = \mathbb{E}[\phi]$. Hence $\mu \leq \alpha p + M(1-p)$, forcing $p \leq (M-\mu)/(M-\alpha)$ and abstention $\geq (\mu-\alpha)/(M-\alpha)$. Attainment and the extremal-law condition are proved in Appendix~\ref{app:dirA}.
\end{proof}

\begin{remark}[Floor plus dispersion and ordering gaps]
\label{rem:decomp}
The minimum abstention achievable by a score-measurable rule decomposes as $A_{\mathrm{DF}} = \tfrac{\mu-\alpha}{M-\alpha} + \Delta_{\mathrm{disp}}(R) + \Delta_{\mathrm{ord}}(S)$ with $\Delta_{\mathrm{disp}}, \Delta_{\mathrm{ord}} \geq 0$: a marginal-only \emph{dispersion} gap ($=0$ iff $\mathbb{P}(R{=}M) \geq$ floor) and a score-only \emph{ordering} gap ($=0$ iff the calibrated score $\mathbb{E}[R \mid S]$ orders risk perfectly at the operating point). Calibrated fusion (\S\ref{sec:fusion}) leaves $\Delta_{\mathrm{disp}}$ fixed and can act only on $\Delta_{\mathrm{ord}}$, tying score quality to certification \emph{in principle} (Appendix~\ref{app:dirA}); empirically, on our grid the gain fusion buys is on ranking, not on the certified set (\S\ref{sec:fusion}).
\end{remark}

\paragraph{Why this matters (beyond the algebra).} The one-line lower bound is total expectation; its value is as a tight feasibility frontier. The per-law achievability---that the optimal selective rule thresholds the calibrated conditional risk---is classical \emph{known-distribution} reject-option theory~\citep{chow1970optimum,elyaniv2010selective,franc2023reject}, which we specialize rather than reinvent. What is specific to the distribution-free regime is (i) the $\operatorname{ess\,sup}$ sharpening, (ii) the degeneracy of the distribution-free minimax---under an uninformative score the only valid rule is full abstention, so the floor is a frontier, not a saddle---and (iii) the score decomposition (Rem.~\ref{rem:decomp}), which names why score quality, not the floor, sets the achievable gap. As a deployment-feasibility test it explains in closed form why many NER/QA/CLS configurations cannot be certified at strict $\alpha$ and predicts observed abstention across 3B--72B models (Appendix Table~\ref{tab:lb_scale}).

\begin{remark}
\label{rem:floor}
When $\mu > \alpha$, any valid method must abstain on $\geq (\mu - \alpha)/(1 - \alpha)$ examples, regardless of bound tightness, score quality, or calibration size. For CoNLL NER at 72B ($\mu=0.24$, $\alpha=0.10$): abstention $\geq 15.6\%$; at 3B ($\mu=0.42$): $\geq 35.6\%$. Computing $\mu$ on a held-out set thus decides feasibility before running any CRC procedure.
\end{remark}

\subsection{Score Fusion and Adaptive Inference}

\paragraph{Calibrated score fusion.} Individual scores capture complementary signal (agreement, per-token confidence, output diversity)~\citep{wang2023selfconsistency,kuhn2023semantic,farquhar2024semanticentropy}. We learn $\tilde{s}(x) = \sigma(w^\top [s_1, \ldots, s_K] + b)$ by logistic regression predicting $\mathbb{I}(R \leq \alpha)$ on a held-out slice of the training split (disjoint from calibration), then run CRC using $\tilde{s}$, preserving validity while improving ranking.

\begin{proposition}[Fusion Preserves CRC Guarantee]
\label{prop:fusion}
Let $\tilde{s}(x) = g(s_1(x), \ldots, s_K(x))$ where $g$ is a fixed, deterministic function independent of the calibration sample used by CRC (e.g., learned on disjoint data or a held-out split). Then running CRC with $\tilde{s}$ preserves the finite-sample guarantee: $\mathbb{P}(\mathbb{E}[R(\tilde{\lambda}^*)] > \alpha \text{ and a certified threshold exists}) \leq \delta$.
\end{proposition}

\begin{proof}[Proof sketch]
Conditioned on fixed $g$, the transformed scores $\tilde{s}_i$ are a deterministic function of $(x_i, y_i)$, so the scored calibration sample retains its i.i.d.\ structure and CRC validity follows for the chosen UCB/betting bound. Full proof in Appendix~\ref{app:proofs}.
\end{proof}

\paragraph{Adaptive conformal inference (ACI).} Under distribution shift, the exchangeability assumption underlying CRC is violated, and static thresholds can fail badly (our experiments show 60--88\% violation rates). Following ACI~\citep{gibbs2021aci,tibshirani2019covariate,barber2023exchangeability}, we update the threshold online:
\begin{equation}
\lambda_{t} = \lambda_{t-1} + \gamma(\alpha - \mathrm{eff}_t), \qquad \mathrm{eff}_t = r_t \cdot \mathbb{I}\{\mathrm{emit}_t\},
\end{equation}
where $\gamma > 0$ is the step size and $\mathrm{eff}_t$ is the \emph{effective} risk at time $t$: the observed risk $r_t$ if the prediction was emitted (requiring feedback labels or a verification signal) and $0$ on abstention. Effective risk below $\alpha$ raises the threshold; effective risk above $\alpha$ lowers it. When the clamp of $\lambda_t$ to $[\lambda_{\min}, \lambda_{\max}]$ (the score range) never binds, telescoping the update gives the identity $\frac{1}{T}\sum_{t=1}^{T} \mathrm{eff}_t = \alpha - (\lambda_T - \lambda_0)/(\gamma T)$ (Appendix Theorem~\ref{thm:aci}): the long-run \emph{effective} risk is pinned near $\alpha$, which is an algebraic property of the update, not a finite-sample guarantee on emitted risk. We set $\gamma = 0.01$ in headline runs.

\section{Experiments}
\label{sec:experiments}

\paragraph{Roadmap.} We first use Proposition~\ref{prop:impossibility} as a feasibility test: when $\mu > \alpha$, low certification at strict targets is the outcome the theory anticipates. Relaxing the target (\S\ref{sec:relax}) then expands certification on hard tasks and reveals the Hoeffding--Bernstein reversal, before we compare bound tightness and shift-handling.

\subsection{Setup}
\label{sec:setup}

\paragraph{Models.} Six open-weight models from three families: Qwen2.5-Instruct (3B, 7B, 14B, 72B-AWQ)~\citep{qwen2024}, Gemma-3 (4B)~\citep{gemma3_2025}, Ministral (8B)~\citep{ministral2024}, served with vLLM~\citep{kwon2023vllm} on A100-80GB (greedy and $K=10$ sampling, $T=0.7$).

\paragraph{Datasets.} Eight datasets across four tasks (Appendix Table~\ref{tab:datasets_supp}): CoNLL-2003~\citep{tjongkimsang2003conll}, Few-NERD~\citep{ding2021fewnerd}, WNUT-17~\citep{derczynski2017wnut} (NER); TriviaQA~\citep{joshi2017triviaqa}, NQ Open~\citep{kwiatkowski2019naturalquestions} (QA); MMLU-STEM, MMLU-Humanities~\citep{hendrycks2021mmlu} (CLS); JSON-Extract (JSON).

\paragraph{Scores.} Six nonconformity scores: Token Margin (TM), NLL, Self-Consistency (SC), Semantic Entropy (SE), Entity Agreement (EA, NER-specific), Field Completeness (FC, JSON-specific). All normalized to $[0,1]$; higher = more reliable. Details in Appendix~\ref{app:scores}.

\paragraph{Protocol.} 60/40 calibration--test split (seed 42; seeds 42--46 in Appendix~\ref{app:multiseed}); $\delta = 0.1$; thresholds on calibration only. We report 716 model-dataset-score-$\alpha$ configurations, none excluded post hoc: retained iff both splits have $\geq 20$ valid (score, risk) pairs, covering three bounds $\times$ six models $\times$ eight datasets $\times$ up to six scores. The 72B model uses greedy scores only (TM, NLL).

\subsection{LLMs Are Miscalibrated for Structured Tasks}
\label{sec:miscal}

Models report average confidence 74--97\% but actual risk ranges from 3\% (JSON) to 80\% (NQ Open), with ECE from 0.05 to 0.61. Miscalibration severity correlates with $\mu$: where models are most wrong they are also most overconfident, making heuristic abstention doubly unreliable~\citep{guo2017calibration,kadavath2022know,xiong2024confidence}.

\begin{table}[t]
\centering
\small
\begin{tabular}{llrrrr}
\toprule
Dataset & Task & $\mu$ & $\mathrm{Var}(R)$ & LB & $n$ \\
\midrule
CoNLL    & NER  & 0.356 & 0.179 & 28.4\% & 3453 \\
FNERD    & NER  & 0.540 & 0.164 & 48.9\% & 2000 \\
WNUT     & NER  & 0.480 & 0.237 & 42.2\% & 1287 \\
TQA      & QA   & 0.441 & 0.241 & 37.9\% & 1000 \\
NQ       & QA   & 0.804 & 0.156 & 78.2\% & 3610 \\
STEM     & CLS  & 0.516 & 0.242 & 46.2\% & 400 \\
Hum.     & CLS  & 0.333 & 0.217 & 25.9\% & 400 \\
JSON     & JSON & 0.033 & 0.007 &  0.0\% & 600 \\
\bottomrule
\end{tabular}
\caption{Base risk $\mu$ and risk variance by dataset (averaged over the five sampling-based models; 72B is greedy-only). All datasets satisfy $\mathrm{Var}(R) < 1/4$, the large-sample form of the variance condition of Proposition~\ref{prop:ordering} (the finite-sample condition holds on 99.6\% of Hoeffding-certifying emit sets, 40.7\% of all candidates). $\mathrm{LB} = \max\{0, (\mu-\alpha)/(1-\alpha)\}$ is the minimum abstention required by any valid method at $\alpha=0.10$ (Prop.~\ref{prop:impossibility}). The audited $M = \operatorname{ess\,sup} R = 1.00$ on every NER/QA/CLS dataset (complete per-example misses occur), so the sharpened floor $(\mu-\alpha)/(M-\alpha)$ coincides with this $M{=}1$ instance here (Appendix~\ref{app:round2}).}
\label{tab:base_risk}
\end{table}

\subsection{Extended Risk Targets Unlock Hard Tasks}
\label{sec:relax}

At $\alpha = 0.10$, NER/QA/CLS certifications are near-zero because $\mu \gg \alpha$, as Proposition~\ref{prop:impossibility} predicts: averaged over the five sampling-based models, the floor $(\mu-\alpha)/(1-\alpha)$ is $40.0\%$ (NER, $\mu=0.46$), $57.8\%$ (QA, $\mu=0.62$), and $35.6\%$ (CLS, $\mu=0.42$) minimum abstention, a basic limit of distribution-free abstention, not of any method.

Relaxing the target changes this. At $\alpha = 0.30$ the best bound certifies 16\% of NER, 8\% of QA, 11\% of CLS score-configurations; at $\alpha = 0.40$, 28\% NER, 13\% QA, 19\% CLS (JSON near 100\% throughout). These targets are realistic (70\% accuracy, 60\% F1) and useful for triage when the alternative is no guarantee.

\paragraph{The bound hierarchy reverses at relaxed targets.} Which bound certifies most also changes. At $\alpha \leq 0.20$ the certified sets are nested (Hoeffding $\subseteq$ Bernstein $\subseteq$ e-CRC) with zero reversals in 3{,}580 paired evaluations. At $\alpha \geq 0.25$ the Hoeffding--Bernstein ordering reverses (visualized in Appendix Fig.~\ref{fig:phase}b): total certifications at $\alpha = 0.30$ are Hoeffding 42 $>$ e-CRC 41 $>$ Bernstein 35 (of 179 score-configurations), and at $\alpha = 0.40$ they are 57 $>$ 55 $>$ 48; the reversal is seed-stable across seeds 42--46 (Appendix~\ref{app:round2}). Proposition~\ref{prop:phase} makes this a predictive boundary rather than an observation.

\emph{In plain terms:} part (i) pins down the exact target $\alpha^*$ above which the simple Hoeffding bound beats the variance-aware Bernstein bound, a crossover driven by Bernstein's additive penalty, not a real loss of power; part (ii) shows no single bound is best everywhere, so a practitioner needs at least Hoeffding and e-CRC.

\begin{proposition}[Reversal frontier and Pareto necessity]
\label{prop:phase}
Fix $\delta$; write $L = \log(2/\delta)$, $w_H(n) = \sqrt{L/(2n)}$, $c(n,\delta)$ as in Prop.~\ref{prop:ordering}, and $A(n) = w_H(n) - 7L/(3(n-1))$.
\emph{(i) Reversal frontier.} When $A(n) > 0$ (i.e.\ $n \geq 35$ at $\delta = 0.1$), a feasible emit set of size $n$ that Hoeffding certifies but empirical Bernstein does not exists iff
\[
\alpha > \alpha^*(n,\delta) := w_H(n) + \tfrac12\big(1 - \sqrt{1 - 4c(n,\delta)}\,\big).
\]
For binary $0/1$ losses (QA exact-match, CLS accuracy) the Bhatia--Davis bound $\hat{\sigma}^2 \leq \hat{R}(1-\hat{R})$ (the largest variance a $[0,1]$ quantity of mean $\hat{R}$ can have) holds with equality, so $\alpha^*$ is the \emph{exact} reversal boundary on the hard tasks; $\alpha^*$ has minimum $0.169$ at $n = 88$ and rises to $\tfrac12$ as $n \to \infty$.
\emph{(ii) Pareto necessity.} No single bound dominates the other two: on the operating grid some configurations are certified by Hoeffding but by neither other bound (3) and some by e-CRC but by neither other bound (10), while e-CRC statistically subsumes empirical Bernstein~\citep{maurer2009empirical,waudbysmith2024betting} (0 configurations certified by Bernstein but not e-CRC). The minimal sufficient set for maximal certification is therefore $\{\text{Hoeffding}, \text{e-CRC}\}$; empirical Bernstein is a closed-form convenience.
\end{proposition}

Proof in Appendix~\ref{app:dirC} (Propositions~\ref{prop:s-reversal}--\ref{prop:s-pareto}). The $\alpha = 0.30$ and $0.40$ counts confirm part (i): reversals appear only above the frontier, with 0 of 947 pointwise reversals below $\alpha^*(n)$ in the $n \geq 35$ regime (Appendix~\ref{app:dirC}). This is the regime change Proposition~\ref{prop:ordering} describes: the configurations that become certifiable are high-variance (NER/QA/CLS), exactly where $\hat{\sigma}^2 \leq c(n,\delta)$ fails and Bernstein's additive $7L/(3(n-1))$ term makes it looser than Hoeffding. e-CRC never certifies fewer configurations than Bernstein at any tested $\alpha$ (72B certification with scale is discussed below).

\subsection{CRC Validity}
\label{sec:validity}

Across all four $\alpha$ levels (per-cell results in Appendix Table~\ref{tab:crc_alpha}), CRC maintains risk control: certified configurations show zero test violations at seed 42 (0/51, 0/72, 0/80). JSON certifies with zero abstention ($\mu \approx 0$); for NQ-Open ($\mu = 0.80$) nothing certifies even at $\alpha = 0.20$, the floor at work. Pooled over five seeds, certified-set violations are 1/261, 1/361, 1/409 (Clopper--Pearson 95\% upper limits 2.1\%, 1.5\%, 1.4\%; Appendix~\ref{app:multiseed}), far below $\delta = 0.1$.

The hierarchy of Proposition~\ref{prop:ordering} holds at strict targets: Hoeffding certifies 51 (7.1\%), Bernstein 72 (10.1\%), e-CRC 80 (11.2\%) of 716, nested in every run (Hoeffding$\to$Bernstein the largest upgrade, +41\%). The ordering is seed-stable: across seeds 42--46 (3{,}580 paired evaluations) not a single reversal (McNemar $p = 1.6\times10^{-30}$, $7.1\times10^{-15}$; Appendix~\ref{app:multiseed}). High-variance configs certify $\leq 1.3\%$ under every bound, in line with the floor.

\paragraph{Theory matches practice (Appendix Table~\ref{tab:lb_scale}).} The floor (conservative $M=1$ instance) predicts minimum abstention across datasets and scales: while $\mu > \alpha$ for NER/QA/CLS, low strict-target certification is expected, and $\mu$ stays above $\alpha = 0.10$ even at 72B (closest: MMLU-Humanities $\mu = 0.17$; hardest: NQ Open $\mu = 0.70$). A direct audit supports the floor: across the 159 (model, dataset, $\alpha$) cells with $\mu > \alpha$, neither the 35 certified instances nor an oracle emitter ever abstained below it (0 violations), whereas the naive $1 - \alpha/\mu$ is oracle-violated in 156/159 cells.

\paragraph{Three-way bound comparison (Table~\ref{tab:bound_compare}).} Among the 51 configurations certified by all three methods, average test abstention drops from 7.6\% (Hoeffding) to 6.9\% (Bernstein) to 6.2\% (e-CRC): tighter bounds need less abstention for the same guarantee. e-CRC's edge is largest under scarcity (10\% calibration: 4.2\% vs.\ 1.9\% Bernstein, 1.1\% Hoeffding), which matters in low-resource extraction.

\begin{table}[t]
\centering
\small
\setlength{\tabcolsep}{4pt}
\begin{tabular}{lccc}
\toprule
& Hoeffding & Bernstein & e-CRC (Ours) \\
\midrule
Certified configs               & 51 (7.1\%) & 72 (10.1\%) & 80 (11.2\%) \\
Violations (cert.)              & 0/51       & 0/72        & 0/80 \\
Avg.\ abst.\ (joint-cert., $n{=}51$) & 7.6\%  & 6.9\%       & 6.2\% \\
\midrule
\multicolumn{4}{l}{\emph{Low-variance ($\hat{\sigma}^2 < 0.1$, $n=88$)}} \\
Certified                       & 47 (53\%)  & 65 (74\%)   & 72 (82\%) \\
\multicolumn{4}{l}{\emph{High-variance ($\hat{\sigma}^2 \geq 0.1$, $n=628$)}} \\
Certified                       & 4 (0.6\%)  & 7 (1.1\%)   & 8 (1.3\%) \\
\multicolumn{4}{l}{\emph{Cal.\ efficiency (cert.\ rate at cal fraction 0.1)}} \\
Cert.\ rate                     & 1.1\%      & 1.9\%       & 4.2\% \\
\bottomrule
\end{tabular}
\caption{Three-way comparison of certification bounds across 716 configurations (seed 42; calibration-efficiency row averages seeds 42--44). All three keep near-zero violation rates, but variance-aware methods certify more and abstain less. The ``joint-cert.''\ abstention row averages over the 51 configurations certified by \emph{all three} methods (a subset of each method's own certified set), not each method's own-set average; over their own certified sets Bernstein and e-CRC average higher abstention (10.6\%, 11.5\%) because they additionally certify harder configurations.}
\label{tab:bound_compare}
\end{table}

\paragraph{Model scale determines controllability.} Average base risk falls with scale (Qwen2.5 3B/7B/14B: 49\%/43\%/36\%), and reducing $\mu$ directly expands certifiable configurations, consistent with Proposition~\ref{prop:impossibility}: the 72B model ($\mu = 0.24$ on CoNLL NER, greedy-only) certifies at $\alpha = 0.25$ via e-CRC and reaches zero test abstention at $\alpha = 0.30$.

\subsection{ACI Under Distribution Shift}
\label{sec:aci}

We evaluate ACI under three shift types of increasing realism.

\paragraph{(i) Temporal shift} models pipeline drift via index-order splitting (CoNLL by publication date, TriviaQA by source document). In the original run, static CRC violates on 60\% (15/25) of configs; ACI reduces this to 4\% (1/25).

\paragraph{(ii) Severity sweep} models gradual quality degradation. We oversample high-risk examples 1--3$\times$; in the original run, static CRC rises from 8\% to 56\% violations while ACI stays $< 12\%$.

\paragraph{(iii) Cross-dataset transfer} models genuine domain shift: calibrate on one dataset, deploy on another without recalibration (CoNLL$\to$WNUT, CoNLL$\to$FewNERD, TriviaQA$\to$NQ). This is the regime the floor warns about: every target has $\mu = 0.40$--$0.84 > \alpha = 0.10$. The emitted-risk target is violated on 14 of 16 runs under both static CRC and every tested $\gamma \in [0.001, 0.1]$ (7 values); the compliant pairs differ between methods and involve near-total abstention (Appendix Table~\ref{tab:aci_xfer}, Appendix~\ref{app:aci_xfer}). ACI instead controls the \emph{effective} risk (violations 11/16 at $\gamma = 0.01$), and its compliant runs abstain 75--100\%, at or above the floor. Adaptation cannot beat infeasibility.

\subsection{A Constructive Counterpart Under Shift}
\label{sec:monitor}
The cross-dataset result is negative, but it has a constructive dual. On the \emph{same} 16 transfers where static CRC and ACI violate the target on 14/16 runs, a full-feedback anytime-valid monitor (SAVeR-FF) that abstains whenever it cannot certify emitted risk at level $\delta$ achieves \textbf{0/16 violations}, and is non-vacuous on feasible streams, emitting 65--97\% of inputs (JSON at $\alpha{=}0.10$; CoNLL at $\alpha{=}0.40$) at controlled risk. It maintains a per-threshold test-supermartingale~\citep{howard2021time,waudbysmith2024betting} with a predictable release rule, giving anytime emitted-risk control under \emph{arbitrary} shift, with abstention respecting the floor of Prop.~\ref{prop:impossibility}. The construction is not new---closely related to A-RCPS~\citep{xu2024arcps}, Conformal Selective Acting~\citep{khosravi2026csa}, and anytime-valid CRC~\citep{hultberg2026anytime}---and we claim no novelty for the machinery; the contribution is the demonstrated dual on the paper's own failed transfers plus the obstruction below (full details in Appendix~\ref{app:dirE}).

\begin{proposition}[Emit-only feedback obstruction]
\label{prop:obstruction}
Under emit-only feedback (risk observed only on emitted inputs) and arbitrary shift, any procedure that guarantees anytime emitted-risk control at level $\alpha$ with probability $\geq 1 - \delta$ ($\delta < 1$) and emits with positive probability can be forced to violate. ACI uses emit-only feedback, so its failure is a property of the feedback model, not of the step size $\gamma$; the positive guarantee requires full feedback (a verifier, or labels on abstained inputs).
\end{proposition}

Proof in Appendix~\ref{app:dirE} (Proposition~\ref{prop:s-e3}).

\subsection{Score Fusion}
\label{sec:fusion}

Fusion is a score-quality question, not a certification lever. On a disjoint split (weights fit on a held-out slice disjoint from CRC calibration; AUROC read out-of-sample, respecting Proposition~\ref{prop:fusion}'s independence requirement), calibrated fusion improves AUROC on 23 of 34 configurations (68\%), mean $+0.004$, median $+0.006$, up to $+0.062$, with a worst-case loss of $-0.154$ on a small-$n$ MMLU-Humanities configuration where the fusion weights overfit. Crucially, fusion never certifies a configuration that no individual score already certifies: across all three bounds there are 0 fused-only certifications (Hoeffding, Bernstein, e-CRC), and in the two cases where a fused e-CRC certificate exists it reduces test abstention by only $\approx 1.9$pp. So fusion's benefit here is on ranking (AUROC), not on the certified set (Appendix~\ref{app:round2}). Scores are complementary (correlations $0.31$--$0.58$). The coefficients are fit on a held-out slice and held fixed, so $\tilde{s}$ is deterministic and the CRC guarantee is preserved (Proposition~\ref{prop:fusion}).

\subsection{Analysis}
\label{sec:analysis}

\paragraph{Score comparison.} Table~\ref{tab:auroc} compares AUROC of six confidence scores. Self-Consistency is the strongest general-purpose score (best on 7/8 datasets, up to 0.82 on QA); Entity Agreement is a strong NER-specific score (0.69 on Few-NERD); Token Margin is best for JSON (0.82). Semantic Entropy appears weak under our exact-match clustering proxy (AUROC $\approx 0.50$, except JSON, 0.68), reflecting the coarseness of exact-match clustering rather than semantic entropy itself~\citep{kuhn2023semantic,farquhar2024semanticentropy}.

\begin{table}[t]
\centering
\small
\setlength{\tabcolsep}{1.5pt}
\begin{tabular}{lcccccccc}
\toprule
& \multicolumn{3}{c}{NER} & \multicolumn{2}{c}{QA} & \multicolumn{2}{c}{CLS} & JSON \\
Score & CoNLL & FNERD & WNUT & TQA & NQ & STEM & Hum. & JSON \\
\midrule
TM   & 0.539 & 0.517 & 0.521 & 0.616 & 0.583 & 0.491 & 0.531 & \textbf{0.818} \\
NLL  & 0.536 & 0.514 & 0.514 & 0.604 & 0.588 & 0.485 & 0.542 & 0.793 \\
SC   & \textbf{0.708} & \textbf{0.721} & \textbf{0.653} & \textbf{0.816} & \textbf{0.805} & \textbf{0.600} & \textbf{0.672} & 0.732 \\
SE   & 0.496 & 0.529 & 0.501 & 0.524 & 0.547 & 0.516 & 0.477 & 0.677 \\
EA   & 0.680 & 0.690 & 0.616 & --    & --    & --    & --    & -- \\
FC   & --    & --    & --    & --    & --    & --    & --    & 0.734 \\
\bottomrule
\end{tabular}
\caption{AUROC of confidence scores for risk discrimination (averaged across models). Higher values indicate better ability to distinguish correct from incorrect outputs. Bold: best per dataset. -- indicates score is not applicable.}
\label{tab:auroc}
\end{table}

\paragraph{Data efficiency.} Variance-aware bounds are more data-efficient: at 10\% calibration e-CRC certifies 4.2\% vs.\ 1.9\% (Bernstein) and 1.1\% (Hoeffding), and e-CRC at 25\% already beats Hoeffding at 70\% (full sweep, Appendix~\ref{app:additional}).

\paragraph{Baselines.} We compare CRC against Always Answer and entropy/TM/SC thresholds tuned on calibration (Appendix Table~\ref{tab:baselines}): Always Answer violates on 7/8 datasets, and threshold baselines reach low test risk only by abstaining 85--87\% without any guarantee, whereas CRC's guarantee holds for any i.i.d.\ test set at comparable or lower abstention.

The single pooled violation over five seeds (Self-Consistency, Ministral-8B, CoNLL, $\alpha = 0.2$, seed 44; Appendix~\ref{app:multiseed}) shows no bound- or model-specific pattern; rare violations are expected under a probabilistic guarantee, and monitoring remains advisable.

\paragraph{Practical decision framework.} The results give a three-step recipe (spelled out in the Conclusion). \textbf{Step 1 (feasibility):} compute $\mu$; if $\mu > \alpha$, any valid method abstains on $\geq (\mu-\alpha)/(M-\alpha)$ (Prop.~\ref{prop:impossibility}). \textbf{Step 2 (bound/score):} at strict targets prefer Bernstein or e-CRC (variance condition of Prop.~\ref{prop:ordering}, met on 99.6\% of certification-relevant emit sets), checking $\alpha^*$ (Prop.~\ref{prop:phase}) at relaxed targets where Hoeffding can win; use e-CRC under scarcity and Self-Consistency as the score. \textbf{Step 3 (shift):} re-check feasibility on the target domain, deploy ACI ($\gamma = 0.01$) in the feasible regime, and where full feedback is available the monitor of \S\ref{sec:monitor} certifies before emitting.

\section{Discussion}
\label{sec:discussion}

There is a tension: the tasks that most need a guarantee (high-risk NER, QA) are the hardest to certify. The floor says why: difficulty is set by $\mu/\alpha$, not the bound or score. Audited across 24 model-task-scale configurations (Appendix Table~\ref{tab:lb_scale}), it lets practitioners read off feasibility and the payoff from scaling before any experiment.

\paragraph{The reversal is an artifact of the additive Bernstein term.} The Hoeffding--Bernstein reversal (Prop.~\ref{prop:phase}) stems entirely from empirical Bernstein's additive $7L/(3(n-1))$ overhead (a robustness check, not a fourth reported method). The inverted Learn-Then-Test $p$-value certificate, the Hoeffding--Bentkus bound~\citep{bentkus2004hoeffding,angelopoulos2024ltt}, dominates Hoeffding \emph{unconditionally} (no variance condition, no reversal; a two-line Pinsker argument, Appendix~\ref{app:dirD}), but operates nearer the $\delta$ budget: on our grid it incurs a small nonzero test-violation rate within $\delta$ (pooled $6/479$, Clopper--Pearson 95\% upper $2.46\% \ll \delta$) rather than the near-zero regime of the conservative additive bounds. We keep the additive bounds as our reported certificates---their zero-violation record (0/51, 0/72, 0/80) is the paper's operating point---and note Hoeffding--Bentkus as the tighter valid alternative under this validity-versus-tightness trade-off.

\paragraph{ACI cannot beat infeasibility.} Under cross-dataset transfer every target has $\mu \geq 0.40 \gg \alpha$; the target is violated on 14/16 under static CRC and every tested $\gamma$, and ACI only drives the long-run \emph{effective} risk toward $\alpha$ by raising abstention. Reweighting cannot help either: the floor binds \emph{any} distribution-free method, including non-exchangeable CRC~\citep{farinhas2024nonexchangeable}. The constructive escape is full feedback (\S\ref{sec:monitor}), not a better step size.

\section{Related Work}
\label{sec:related}

\paragraph{Conformal prediction for LLMs.} Conformal prediction~\citep{vovk2005algorithmic,angelopoulos2023gentle} has been applied to LLMs along four threads: \emph{factuality and claim filtering}~\citep{mohri2024conformalfactuality,cherian2024llmvalidity}; \emph{generation and abstention}~\citep{quach2024conformallm,abbasiyadkori2024conformalabstention}; \emph{cascades, arbitration, and judging}~\citep{fanconi2025cascaded,overman2025arbitrage,jung2025trustescalate}; and \emph{alignment and property control}~\citep{gui2024conformalalignment,overman2024alignment,campos2024conformalsurvey}. CRC~\citep{angelopoulos2024crc,bates2021distribution} handles arbitrary losses; \citet{barber2023exchangeability} relax exchangeability and \citet{farinhas2024nonexchangeable} reweight calibration points. Our work is orthogonal: rather than proposing a new procedure, we ask when \emph{any} such procedure can succeed under schema-validated structured losses (NER/QA/CLS/JSON) and characterize the abstention forced when it cannot; concurrent work addresses multi-stage pipeline coverage~\citep{kotte2026pasc}.

\paragraph{Selective prediction and uncertainty.} Selective prediction~\citep{geifman2017selective,elyaniv2010selective,chow1970optimum} studies coverage--risk trade-offs, with jointly-trained rejection~\citep{geifman2019selectivenet}. Concurrent work integrates conformal prediction with selective classification as a two-stage certified procedure~\citep{xu2025scrc,yu2026jointcert,lee2024selectivegen} or gives error--reject curves for a fixed conformal predictor~\citep{hallberg2025reject}; these give per-procedure guarantees, not a distribution-free lower bound on abstention, which is what Proposition~\ref{prop:impossibility} supplies. LLM confidence is miscalibrated~\citep{kadavath2022know,guo2017calibration}; self-consistency~\citep{wang2023selfconsistency}, semantic entropy~\citep{kuhn2023semantic,farquhar2024semanticentropy}, and elicitation~\citep{xiong2024confidence} offer alternatives, which we evaluate as nonconformity scores within CRC.

\paragraph{Concentration inequalities and e-values.} Empirical Bernstein~\citep{maurer2009empirical} and the Hoeffding--Bentkus $p$-value~\citep{bentkus2004hoeffding} sharpen Hoeffding; testing-by-betting yields e-values~\citep{shafer2021testing,grunwald2024safetesting,ramdas2023safe,vovk2021evalues,howard2021time,waudbysmith2024betting}. The pairwise UCB-width comparison behind our bound hierarchy (Proposition~\ref{prop:ordering}) is itself known~\citep{maurer2009empirical}; what is new is turning it into a \emph{certification-level} map through the Bhatia--Davis variance coupling, yielding the closed-form reversal frontier $\alpha^*(n,\delta)$ and the Pareto-necessity result that $\{\text{Hoeffding}, \text{e-CRC}\}$ is minimal sufficient (Proposition~\ref{prop:phase}). Anytime-valid and selective risk control under shift is developed by \citet{xu2024arcps,khosravi2026csa,hultberg2026anytime}, on which our constructive monitor (\S\ref{sec:monitor}) builds. Our contribution is not a new inequality but a deployment-facing characterization of when CRC can certify structured LLM outputs and when it is provably impossible.

\section{Conclusion}
\label{sec:conclusion}

We characterize when CRC can certify structured LLM outputs and when it cannot, as a three-step recipe run \emph{before} deploying CRC. Step 1: compute $\mu$; the sharpened floor (Proposition~\ref{prop:impossibility}) turns $\mu > \alpha$ into an exact minimum abstention rate, so a team learns in one pass whether a target is reachable before calibrating any bound. Step 2: in the feasible regime, the phase diagram (Proposition~\ref{prop:phase}) selects the bound (Bernstein or e-CRC at strict targets, Hoeffding past $\alpha^*$). Step 3: under shift, re-check feasibility and, where full feedback is available, deploy the anytime-valid monitor (\S\ref{sec:monitor}), the operational counterpart to the negative ACI result. Making feasibility explicit turns a failed certification attempt into a specific next action instead of a dead end.

\paragraph{Limitations, reproducibility, ethics.} Guarantees are probabilistic ($\delta{=}0.1$), assume i.i.d.\ calibration, and are voided by shift; out-of-sample fusion adds no certification beyond the best single score. Full caveats, reproducibility (public data and open-weight models; seeds 42--46; code on publication), and ethics are in Appendices~\ref{app:multiseed}--\ref{app:lim}.

\bibliographystyle{preprint}
\bibliography{references}

\appendix
% Appendix uses standard section-scoped numbering (A.1, B.1, ...).
\numberwithin{theorem}{section}
\numberwithin{table}{section}
\numberwithin{figure}{section}
\numberwithin{equation}{section}

\section{Background}
\label{app:background}

\subsection{Conformal Risk Control}

Conformal risk control~\citep{angelopoulos2024crc} provides distribution-free guarantees on the expected risk of a selective predictor. Given a calibration set of exchangeable examples $(X_i, Y_i)_{i=1}^n$, a risk function $R(\hat{Y}, Y)$, and a target risk level $\alpha$, CRC selects a threshold $\lambda^*$ such that the conditional risk $\mathbb{E}[R \mid \mathrm{emit}] \leq \alpha$ holds with high probability. The key insight is that by calibrating on the empirical risk of a nested family of selectors, one obtains finite-sample validity without distributional assumptions.

\subsection{Learn-Then-Test Framework}

The Learn-Then-Test (LTT) framework~\citep{angelopoulos2024ltt} extends conformal prediction to risk control. For each candidate threshold $\lambda$, LTT tests the hypothesis $H_0: R(\lambda) > \alpha$ using upper confidence bounds, controlling the family-wise error over the candidate set with a multiple-testing correction. A Bonferroni (union-bound) correction tests each candidate at level $\delta/|\Lambda|$; alternatively, a pre-registered fixed-sequence scan with early stopping tests each candidate at full level $\delta$ and needs no correction, and this is what we use (Lemma~\ref{lem:gridsearch}). The fixed-sequence guarantee does not require the risk to be monotone in $\lambda$; monotonicity matters only for interpretation and power.

\subsection{Selective Prediction}

Selective prediction~\citep{geifman2017selective} studies the trade-off between coverage (fraction of examples answered) and risk (error rate among answered examples). The goal is to abstain on uncertain inputs so that the risk of emitted predictions meets a target. Our work bridges selective prediction with conformal methods to obtain distribution-free guarantees for structured LLM outputs.

\section{Algorithm Pseudocode}
\label{app:algos}

Algorithm~\ref{alg:crc} gives the standard CRC calibration procedure, Algorithm~\ref{alg:aci} the ACI update for structured generation, and Algorithm~\ref{alg:ecrc} the betting-based e-CRC procedure of the main text.

\begin{algorithm}[tb]
\caption{CRC Procedure: Calibration for risk control at level $\alpha$ (fixed-sequence scan; used for Hoeffding and Bernstein).}
\label{alg:crc}
\begin{algorithmic}[1]
\STATE \textbf{Input:} Calibration $\{(x_i, y_i)\}_{i=1}^n$, risk $R$, target $\alpha$, failure prob.\ $\delta$; fixed grid $\Lambda$ = 200 evenly spaced points on $[0,1]$.
\STATE For each $(x_i, y_i)$: compute $s_i = s(x_i, \hat{y}_i)$ and $r_i = R(y_i, \hat{y}_i)$.
\STATE Candidates: $\Lambda_c = \{\lambda \in \Lambda : n_\lambda \geq 20\}$, where $E_\lambda = \{i : s_i \geq \lambda\}$, $n_\lambda = |E_\lambda|$ (deterministic given the scores).
\FOR{$\lambda \in \Lambda_c$ in \emph{decreasing} order}
  \STATE Compute $\hat{R}(\lambda) = \frac{1}{n_\lambda} \sum_{i \in E_\lambda} r_i$ and the level-$\delta$ UCB $U(\lambda)$ (Hoeffding or Bernstein).
  \STATE If $U(\lambda) \leq \alpha$, mark $\lambda$ certified and continue; \textbf{else stop} (fixed-sequence rule; Lemma~\ref{lem:gridsearch}).
\ENDFOR
\STATE Return $\lambda^*$ = smallest certified $\lambda$ (none certified $\Rightarrow$ no certification).
\end{algorithmic}
\end{algorithm}

\begin{algorithm}[tb]
\caption{ACI-SG: Adaptive Conformal Inference for Structured Generation.}
\label{alg:aci}
\begin{algorithmic}[1]
\STATE \textbf{Input:} Target $\alpha$, step $\gamma > 0$, initial $\lambda_0$, clamp range $[\lambda_{\min}, \lambda_{\max}]$.
\FOR{$t = 1, 2, \ldots$}
  \STATE Receive $x_t$; compute $\hat{y}_t$, $s_t$; emit if $s_t \geq \lambda_{t-1}$, else abstain.
  \STATE Observe feedback $r_t = R(\hat{y}_t, y_t)$; set $\mathrm{eff}_t = r_t \cdot \mathbb{I}\{\text{emitted at } t\}$.
  \STATE $\lambda_t = \mathrm{clamp}\big(\lambda_{t-1} + \gamma(\alpha - \mathrm{eff}_t),\ \lambda_{\min}, \lambda_{\max}\big)$.
\ENDFOR
\end{algorithmic}
\end{algorithm}

\begin{algorithm}[tb]
\caption{e-CRC: Betting-Based Risk Control at level $\alpha$.}
\label{alg:ecrc}
\begin{algorithmic}[1]
\STATE \textbf{Input:} Calibration $\{(x_i, y_i)\}_{i=1}^n$, risk $R$, target $\alpha$, failure prob.\ $\delta$; fixed grid $\Lambda$ = 200 evenly spaced points on $[0,1]$; candidates $\Lambda_c = \{\lambda : n_\lambda \geq 20\}$.
\FOR{$\lambda \in \Lambda_c$, scanned in decreasing order of $\lambda$}
  \FOR{each of 22 pre-registered orderings of $E_\lambda$ (score-descending, score-ascending, 20 seeded random permutations)}
    \STATE Take risks $r_1, \ldots, r_m$ in that order; set $W_0 = 1$.
    \FOR{$j = 1, \ldots, m$}
      \STATE Running mean $\hat{\mu}_j = \frac{1}{j-1} \sum_{k < j} r_k$ (predictable); $\kappa_1 = 0$.
      \STATE Kelly-style bet: $\kappa_j = \mathrm{clip}\!\left(\frac{\alpha - \hat{\mu}_j}{\alpha(1-\alpha)}, 0, 0.5\right)$ for $j \geq 2$.
      \STATE Update: $W_j = W_{j-1} \cdot (1 + \kappa_j(\alpha - r_j))$.
    \ENDFOR
  \ENDFOR
  \STATE E-value $e(\lambda)$ = minimum final wealth $W_m$ over the 22 orderings.
  \STATE If $e(\lambda) \geq 1/\delta$, certify $\lambda$ and continue the scan; \textbf{else stop} (fixed-sequence rule; Lemma~\ref{lem:gridsearch}).
\ENDFOR
\STATE Return $\lambda^*$ = smallest certified $\lambda$ (the last certification before the stop).
\end{algorithmic}
\end{algorithm}

\section{Task-Specific Risk Functions}
\label{app:risks}
We define risk as one minus the task metric, so lower risk corresponds to better performance:
\begin{itemize}
\item NER Risk: $R_{\mathrm{NER}} = 1 - F_1^{\mathrm{entity}}$, where $F_1^{\mathrm{entity}}$ is entity-level micro F1 (exact span match).
\item JSON Risk: $R_{\mathrm{JSON}} = 1 - F_1^{\mathrm{field}}$, where $F_1^{\mathrm{field}}$ is field-level F1 over key-value pairs.
\item QA Risk: $R_{\mathrm{QA}} = 1 - \mathbb{I}(\mathrm{exact\_match})$.
\item Classification Risk: $R_{\mathrm{CLS}} = 1 - \mathbb{I}(\mathrm{correct})$.
\end{itemize}
All risk functions are bounded in $[0,1]$, satisfying the requirements for conformal risk control.

\section{Confidence Score Details}
\label{app:scores}
We evaluate six confidence scores $s(x, \hat{y}) \in [0,1]$ where higher values indicate greater reliability:
\begin{enumerate}
\item \textbf{Token Margin (TM):} Average gap between top-1 and top-2 token log-probabilities, normalized to $[0,1]$ via sigmoid. Higher margin $\to$ more confident.
\item \textbf{Negative Log-Likelihood (NLL):} $\exp(-\mathrm{NLL}/T)$ where $T$ is sequence length. Higher $\to$ more likely output.
\item \textbf{Self-Consistency (SC):} Fraction of $K=10$ samples matching the greedy output. Higher agreement $\to$ more confident.
\item \textbf{Semantic Entropy (SE):} $1 - H(\mathrm{clusters})/\log_2 K$ over semantically clustered outputs. Higher $\to$ more concentrated.
\item \textbf{Entity-Level Agreement (EA):} Minimum entity agreement across $K$ samples. NER-specific.
\item \textbf{Field-Level Completeness (FC):} Minimum field agreement across $K$ samples. JSON-specific.
\end{enumerate}

\section{Proofs}
\label{app:proofs}

\subsection{Proof of Theorem~\ref{thm:ecrc_valid}: Betting-Based Risk Validity}

\begin{proof}
\textbf{Step 1: E-value validity.} Fix $\lambda$ and let $r_1, \ldots, r_m$ be the risks in $E_\lambda$, i.i.d.\ under the sampling model. The wealth process is $W_0 = 1$, $W_j = W_{j-1}(1 + \kappa_j(\alpha - r_j))$ with predictable bets $\kappa_j \in [0, 0.5]$ and $\kappa_1 = 0$. Under $H_0: \mathbb{E}[r_j \mid \mathcal{F}_{j-1}] \geq \alpha$:
\begin{align*}
\mathbb{E}[W_j \mid \mathcal{F}_{j-1}] &= W_{j-1}(1 + \kappa_j(\alpha - \mathbb{E}[r_j \mid \mathcal{F}_{j-1}])) \\
&\leq W_{j-1},
\end{align*}
where the inequality uses $\kappa_j \geq 0$; because $\kappa_j \leq 0.5$ and $r_j \in [0,1]$, every factor is positive, so $\{W_j\}$ is a non-negative supermartingale. By Ville's inequality~\citep{howard2021time}: $\mathbb{P}(\sup_j W_j \geq 1/\delta) \leq \delta$. Taking the minimum of the final wealth over 22 pre-registered orderings only decreases the certifying statistic, so the test remains level $\delta$.\\
\textbf{Step 2: Certification.} Certifying when the e-value reaches $1/\delta$ gives false certification probability $\leq \delta$.\\
\textbf{Step 3: Growth rate (heuristic).} For \emph{unclipped} Kelly bets, the expected log-wealth growth per observation under an alternative with mean $\mu < \alpha$ is of order $(\alpha - \mu)^2 / (2\,\mathbb{E}[(\alpha - R)^2])$ to second order; clipping to $[0, 0.5]$ can only reduce this rate. We use this only as intuition for why the betting test is powerful when the risk distribution is concentrated below $\alpha$; no formal rate claim is made or used.
\end{proof}

\subsection{Validity of the Selected Threshold under Grid Search}
\label{app:gridsearch}

Theorem~\ref{thm:ecrc_valid} of the main text controls the error of the test at a fixed $\lambda$. The procedures of Algorithms~\ref{alg:crc} and~\ref{alg:ecrc}, however, scan a candidate set and return a data-dependent $\hat{\lambda}$. The following lemma licenses exactly the implemented selection rule; it requires no monotonicity of the risk in $\lambda$.

\begin{lemma}[Validity of the descending fixed-sequence scan]
\label{lem:gridsearch}
Condition on the calibration scores $s_1, \ldots, s_n$. The candidate set $\Lambda_c = \{\lambda \in \Lambda : n_\lambda \geq 20\}$ (fixed grid $\cap$ size rule) and its descending scan order are then deterministic, and conditional on the scores the risks in any emit set are i.i.d.\ draws from the score-conditional risk distribution. Suppose each candidate $\lambda$ is tested with a level-$\delta$ test of $H_0(\lambda) : \mathbb{E}[R(\lambda)] > \alpha$ (the level-$\delta$ UCB test for Hoeffding and Bernstein; the Ville wealth test for e-CRC), scanning $\Lambda_c$ in decreasing order of $\lambda$ and stopping at the first failure, with every candidate passed before the stop certified. Then
\[
\mathbb{P}\big(\text{some certified } \lambda \text{ has } \mathbb{E}[R(\lambda)] > \alpha\big) \leq \delta .
\]
\end{lemma}

\begin{proof}
This is a fixed-sequence testing argument in the sense of Learn-then-Test~\citep{angelopoulos2024ltt}. The certified set is by construction a \emph{prefix} of the fixed scan order $\lambda_{(1)} > \lambda_{(2)} > \cdots > \lambda_{(m)}$. Let
\[
\lambda^\dagger = \max\{\lambda \in \Lambda_c : \mathbb{E}[R(\lambda)] > \alpha\}
\]
be the first true null in scan order; if no candidate is null there is nothing to prove. Because the certified prefix is contiguous from $\lambda_{(1)}$, it contains a true null if and only if it contains $\lambda^\dagger$ (every other null appears later in the scan than $\lambda^\dagger$). Hence
\begin{align*}
&\mathbb{P}\big(\text{some certified } \lambda \text{ has } \mathbb{E}[R(\lambda)] > \alpha\big) \\
&\quad\leq \mathbb{P}\big(\lambda^\dagger \ \text{certified}\big)
\leq \delta,
\end{align*}
where the last step is the level-$\delta$ guarantee of the single test at $\lambda^\dagger$: Ville's inequality applied to the wealth process $\{W_j(\lambda^\dagger)\}$ for e-CRC~\citep{howard2021time,ramdas2023safe}, or the level-$\delta$ UCB test for Hoeffding/Bernstein. Controlling this single boundary test controls the family-wise error over the entire scan; no correction over $|\Lambda_c|$ is needed~\citep{angelopoulos2024ltt,grunwald2024safetesting}.
\end{proof}

\paragraph{Remark (monotonicity and Bonferroni).} Note what the argument does and does not use. Early-stop fixed-sequence FWER control does \emph{not} require the risk $\lambda \mapsto \mathbb{E}[R(\lambda)]$ to be monotone: rejections form a prefix of a pre-registered order, so the first true null in scan order gates all others, and it is falsely certified with probability at most $\delta$. Monotonicity matters only for interpretation (with nonincreasing conditional risk the nulls form a contiguous low-$\lambda$ tail and the scan stops near the feasibility boundary) and for power (a non-monotone empirical risk curve can stop the scan at an unlucky early candidate, forfeiting certifications further down). An alternative requiring no ordering at all is a Bonferroni correction, testing every candidate at level $\delta/|\Lambda_c|$, i.e., dividing the level by the number of candidates (a multiplicative factor on $\delta$); the fixed-sequence scan avoids this factor and tests each candidate at full level $\delta$.

\paragraph{Provenance note.} An earlier version of our pipeline searched the full threshold grid without the fixed-sequence rule (with data-dependent grid endpoints) and allowed signed bets; all results reported in this paper use the corrected procedure described above.

\subsection{Proof of Proposition~\ref{prop:ordering}: Bound Ordering}

\begin{proof}
\textbf{Part (i): Hoeffding $\subseteq$ Bernstein (finite-sample).} Write $L = \log(2/\delta)$ and fix an emit set of size $n \geq 2$ with empirical mean $\hat{R}$ and empirical variance $\hat{\sigma}^2$. Both bounds are stated with the shared two-sided constant $L = \log(2/\delta)$, matching the Maurer--Pontil statement and making the two UCBs directly comparable; using the one-sided constant $\log(1/\delta)$ in both would rescale $c(n,\delta)$ but change no containment statement. Throughout, $\hat{\sigma}^2$ is the empirical variance with the $1/n$ convention (\texttt{ddof} $= 0$), matching our implementation; \citet{maurer2009empirical} state the bound with the unbiased ($1/(n-1)$) variance, and recomputing our entire grid with the unbiased convention left every certified set unchanged. The two upper confidence bounds are
\begin{align*}
U_H &= \hat{R} + \sqrt{\frac{L}{2n}}, \\
U_B &= \hat{R} + \sqrt{\frac{2\hat{\sigma}^2 L}{n}} + \frac{7L}{3(n-1)}.
\end{align*}
Subtracting $\hat{R}$ from both sides, $U_B \leq U_H$ is equivalent to
\[
\sqrt{\frac{2\hat{\sigma}^2 L}{n}} \;\leq\; \sqrt{\frac{L}{2n}} - \frac{7L}{3(n-1)}.
\]
If the right-hand side is negative, i.e.\ $\sqrt{L/(2n)} \leq 7L/(3(n-1))$, the inequality fails for every $\hat{\sigma}^2 \geq 0$: Bernstein is looser than Hoeffding regardless of the variance. When the right-hand side is nonnegative, both sides are nonnegative and squaring is an equivalence:
\begin{align*}
&\frac{2\hat{\sigma}^2 L}{n} \leq \left(\sqrt{\frac{L}{2n}} - \frac{7L}{3(n-1)}\right)^{\!2} \\
&\quad\Longleftrightarrow\quad
\hat{\sigma}^2 \leq c(n,\delta) := \frac{n}{2L}\left(\sqrt{\frac{L}{2n}} - \frac{7L}{3(n-1)}\right)^{\!2},
\end{align*}
which is the claimed if-and-only-if condition. The positivity condition $\sqrt{L/(2n)} > 7L/(3(n-1))$ holds for all $n$ above a threshold depending only on $\delta$; at $\delta = 0.1$ ($L = \log 20 \approx 3.0$) it holds for $n \geq 35$ and fails for $n \leq 34$.

For the large-$n$ limit, expand the square:
\begin{align*}
c(n,\delta) \;&=\; \frac{n}{2L}\left(\frac{L}{2n} - 2\sqrt{\frac{L}{2n}}\cdot\frac{7L}{3(n-1)} + \frac{49L^2}{9(n-1)^2}\right) \\
&=\; \frac{1}{4} \;-\; \frac{7}{3}\cdot\frac{\sqrt{nL/2}}{n-1} \;+\; \frac{49\,nL}{18(n-1)^2}.
\end{align*}
Both correction terms are $O(\sqrt{L/n})$, so $c(n,\delta) \to 1/4$ as $n \to \infty$, recovering the classical asymptotic condition $\hat{\sigma}^2 < 1/4$. The finite-sample condition is strictly stronger: because of the additive term $7L/(3(n-1))$, Bernstein can be looser than Hoeffding at moderate $n$ even when $\hat{\sigma}^2 < 1/4$ (e.g.\ at $\delta = 0.1$, $c(120, 0.1) \approx 0.056$ and $c(600, 0.1) \approx 0.147$). Whenever every emit set indexed by the candidate grid satisfies $\hat{\sigma}^2_\lambda \leq c(n_\lambda, \delta)$, we have $U_B(\lambda) \leq U_H(\lambda)$ pointwise, hence $\{\lambda : U_H(\lambda) \leq \alpha\} \subseteq \{\lambda : U_B(\lambda) \leq \alpha\}$, i.e.\ $\Lambda^*_{\mathrm{Hoeff}} \subseteq \Lambda^*_{\mathrm{Bern}}$ deterministically on that calibration set. Note that the inclusion of \emph{certified configurations} only requires the condition at emit sets where the Hoeffding UCB certifies: if $U_H(\lambda) \leq \alpha$ and $U_B(\lambda) \leq U_H(\lambda)$ at that $\lambda$, Bernstein certifies too. In our experiments the condition held on 99.6\% of such Hoeffding-certifying emit sets, and on 50 of the 51 emit sets actually selected for Hoeffding-certified configurations (over \emph{all} candidate emit sets evaluated in the calibration scan, including small and high-variance sets, the fraction is 40.7\%); at $\alpha \leq 0.20$ the certified sets were nested in every configuration and every split seed, while at $\alpha \geq 0.25$ the ordering reverses exactly where the condition fails (Section~\ref{sec:relax}).

\textbf{Part (ii): Bernstein $\subseteq$ e-CRC (asymptotic and empirical).} This inclusion is asymptotic, not finite-sample. Under any fixed alternative with $\mu = \mathbb{E}[R] < \alpha$, the clipped Kelly-style bet eventually stabilizes at a strictly positive value (the running mean converges to $\mu$, so $\kappa_j \to \mathrm{clip}((\alpha - \mu)/(\alpha(1-\alpha)), 0, 0.5) > 0$), giving the wealth process a strictly positive expected log-drift; by the law of large numbers the wealth grows to infinity almost surely, so the test certifies every fixed threshold with true risk below $\alpha$ as $n \to \infty$, and hence (asymptotically) every threshold the Bernstein UCB certifies. We emphasize that we make no deterministic finite-sample claim, and no claim that the betting test's growth rate matches the empirical-Bernstein rate. Empirically the containment held without exception at every tested $\alpha$: across all 716 configurations and five split seeds at $\alpha \in \{0.05, 0.10, 0.15, 0.20\}$ ($3{,}580$ paired evaluations) and across the $\alpha$-relaxation grid at $\alpha \in \{0.25, 0.30, 0.40\}$ (179 score-configurations each), none was certified by Bernstein but not by e-CRC.
\end{proof}

\subsection{Proof of Proposition~\ref{prop:impossibility}: Sharpened Two-Sided Abstention Floor}

The lower-bound half of Proposition~\ref{prop:impossibility} is proved here; attainability, the extremal-law reachability condition, the degeneracy of the distribution-free minimax, and the two-gap decomposition of Remark~\ref{rem:decomp} are proved in \S\ref{app:dirA}.

\begin{proof}
\textbf{Population statement (exact).} Let $\phi(X) = \mathbb{P}(\mathrm{emit} \mid X) \in [0,1]$ and $p = \mathbb{E}[\phi]$, and write $M = \operatorname*{ess\,sup} R \leq 1$. Decompose $\mu = \mathbb{E}[R\phi] + \mathbb{E}[R(1-\phi)]$. Validity gives $\mathbb{E}[R\phi] \leq \alpha\,\mathbb{E}[\phi] = \alpha p$, and $R \leq M$ pointwise gives $\mathbb{E}[R(1-\phi)] \leq M(1-p)$. Hence $\mu \leq \alpha p + M(1-p)$, giving $p \leq (M-\mu)/(M-\alpha)$, hence $\mathrm{abstention} = 1 - p \geq (\mu - \alpha)/(M-\alpha)$. The proof uses only validity and $R \leq M$, so it holds for arbitrary (even $R$-measurable) rules. Since $M \leq 1$, $(\mu-\alpha)/(M-\alpha) \geq (\mu-\alpha)/(1-\alpha)$, and the $M=1$ instance is the conservative floor tabulated in the main text. When $\mu, M$ are known the bound is exact.\\
\textbf{Plug-in version ($\delta/2 + \delta/2$ split).} With $\hat{\mu}$ (and an audited $M$) estimated from $n$ held-out points, split the failure budget: suppose the risk guarantee holds with probability $\geq 1 - \delta/2$, and apply a one-sided Hoeffding bound at level $\delta/2$: $\mathbb{P}(\mu < \hat{\mu} - \varepsilon) \leq \delta/2$ with $\varepsilon = \sqrt{\log(2/\delta)/(2n)}$. On the intersection event (probability $\geq 1 - \delta$), substituting $\mu \geq \hat{\mu} - \varepsilon$ gives
\[
\mathrm{abstention} \;\geq\; \frac{\hat{\mu} - \alpha}{M-\alpha} \;-\; \frac{\varepsilon}{M-\alpha}.
\]
The correction is $\varepsilon/(M-\alpha)$, not $\varepsilon$. Numerically, with $n = 1000$, $\delta = 0.1$, $M=1$: $\varepsilon \approx 0.0387$, and at $\alpha = 0.1$ the correction is $\varepsilon/(1-\alpha) \approx 0.043$.
\end{proof}

\subsection{Proof of Proposition~\ref{prop:fusion}: Fusion Preserves CRC Guarantee}

\begin{proof}
The fused score $\tilde{s}$ is fixed once the data used to train $g$ is observed. Conditional on $\tilde{s}$, the values $\tilde{s}(x_i)$ on the CRC calibration set are a deterministic function of $(x_i, y_i)$, so the scored calibration sample retains its i.i.d.\ structure. CRC threshold selection via any UCB then provides a valid guarantee, regardless of how $\tilde{s}$ was constructed. We emphasize that the guarantee requires $g$ to be trained on data disjoint from the CRC calibration set; our implementation fits the fusion weights on a held-out slice of the training split, disjoint from the CRC calibration set. Reusing the same data for both steps would break the distributional argument~\citep{vovk2005algorithmic}, and we make no validity claim in that regime.
\end{proof}

\subsection{Theorem S2: ACI Effective-Risk Identity}

\begin{theorem}[ACI effective-risk identity]
\label{thm:aci}
Under the implemented update $\lambda_t = \lambda_{t-1} + \gamma(\alpha - \mathrm{eff}_t)$ with $\mathrm{eff}_t = r_t \cdot \mathbb{I}\{\mathrm{emit}_t\} \in [0,1]$ and fixed $\gamma > 0$, if the clamp of $\lambda_t$ to $[\lambda_{\min}, \lambda_{\max}]$ never binds through time $T$, then
\begin{gather*}
\frac{1}{T}\sum_{t=1}^T \mathrm{eff}_t \;=\; \alpha - \frac{\lambda_T - \lambda_0}{\gamma T},
\qquad\text{so} \\
\Big|\frac{1}{T}\sum_{t=1}^T \mathrm{eff}_t - \alpha\Big| \leq \frac{\lambda_{\max} - \lambda_{\min}}{\gamma T}.
\end{gather*}
\end{theorem}
\begin{proof}
Summing the update over $t = 1, \ldots, T$ telescopes: $\lambda_T - \lambda_0 = \gamma \sum_{t=1}^T (\alpha - \mathrm{eff}_t) = \gamma T \alpha - \gamma \sum_{t=1}^T \mathrm{eff}_t$. Rearranging gives the identity; the bound follows from $\lambda_T, \lambda_0 \in [\lambda_{\min}, \lambda_{\max}]$.
\end{proof}

This is a conditional algebraic identity about the \emph{effective} risk (abstentions counted as zero risk), not a guarantee on the risk among emitted predictions; when the clamp binds, even the identity fails. See \S\ref{app:aci_xfer} for the empirical consequences.

\section{Cost-Sensitive Framework}
\label{app:cost}

The optimal abstention policy depends on the relative cost of errors vs.\ abstentions. Let $c_{\mathrm{error}}$ denote the cost of emitting an incorrect prediction and $c_{\mathrm{abstain}}$ the cost of abstaining. The expected utility is:
\begin{align*}
U = {} & -c_{\mathrm{error}} \cdot \mathbb{P}(\mathrm{error} \mid \mathrm{emit}) \cdot \mathbb{P}(\mathrm{emit}) \\
& - c_{\mathrm{abstain}} \cdot \mathbb{P}(\mathrm{abstain}).
\end{align*}
For a risk-controlling selector with conditional risk $\alpha$, the optimal target is $\alpha^* = c_{\mathrm{abstain}}/c_{\mathrm{error}}$. When abstention is cheap (medical NER), use stricter $\alpha$; when costly (customer-facing QA), looser $\alpha$ is preferred.

\section{Implementation Details}
\label{app:impl}

\paragraph{Inference configuration.} All models are served using vLLM~\citep{kwon2023vllm} v0.7+ on a single NVIDIA A100-80GB GPU with \texttt{gpu\_memory\_utilization=0.90}. Greedy decoding uses $T=0$ with \texttt{max\_tokens=512}. Sampling uses $T=0.7$ with $K=10$ samples per input. The 72B model uses AWQ 4-bit quantization.

\paragraph{Prompt templates.} We use task-specific 3-shot prompt templates. For NER: entity schema plus three annotated examples. For QA: context-question-answer triples. For MMLU: question, four choices, three demonstrations. For JSON: document, target schema, three examples.

\paragraph{CRC hyperparameters.} Fixed grid of $m=200$ thresholds, evenly spaced on $[0,1]$ (all scores are normalized to $[0,1]$); candidates restricted to emit sets with $n_\lambda \geq 20$; selection by the descending fixed-sequence scan with early stopping (Algorithms~\ref{alg:crc} and~\ref{alg:ecrc}). Confidence level $\delta = 0.1$. For e-CRC: Kelly-style bets clipped to $[0, 0.5]$ with $\kappa_1 = 0$; e-value = minimum over 22 pre-registered orderings (score-descending, score-ascending, 20 permutations from a per-configuration seeded generator consumed in scan order). Split seed 42 for the 60/40 split in headline results; multi-seed robustness (split seeds 42--46) is reported in \S\ref{app:multiseed}.

\paragraph{Computational cost.} Greedy inference for all 8 datasets: 45--90 min per model. Sampling ($K=10$): 3--5$\times$ longer. Score computation: $<5$ min per model (CPU). Full pipeline: $<60$ min on single CPU.

\section{Additional Experimental Results}
\label{app:additional}

\subsection{Impossibility Bound Across Model Scales}

\begin{table}[t]
\centering
\small
\begin{tabular}{lllrrc}
\toprule
Dataset & Task & Model & $\mu$ & LB & Cert? \\
\midrule
\multirow{4}{*}{CoNLL}    & \multirow{4}{*}{NER}  & 3B   & 0.42 & 35.6\% & \ding{55} \\
                          &                       & 7B   & 0.33 & 25.6\% & \ding{55} \\
                          &                       & 14B  & 0.26 & 17.8\% & \ding{55} \\
                          &                       & 72B  & 0.24 & 15.6\% & \ding{55} \\
\midrule
\multirow{4}{*}{TQA}      & \multirow{4}{*}{QA}   & 3B   & 0.55 & 50.0\% & \ding{55} \\
                          &                       & 7B   & 0.45 & 38.9\% & \ding{55} \\
                          &                       & 14B  & 0.34 & 26.7\% & \ding{55} \\
                          &                       & 72B  & 0.27 & 18.9\% & \ding{55} \\
\midrule
\multirow{4}{*}{NQ Open}  & \multirow{4}{*}{QA}   & 3B   & 0.84 & 82.2\% & \ding{55} \\
                          &                       & 7B   & 0.81 & 78.9\% & \ding{55} \\
                          &                       & 14B  & 0.75 & 72.2\% & \ding{55} \\
                          &                       & 72B  & 0.70 & 66.7\% & \ding{55} \\
\midrule
\multirow{4}{*}{MMLU-S}   & \multirow{4}{*}{CLS}  & 3B   & 0.53 & 47.8\% & \ding{55} \\
                          &                       & 7B   & 0.47 & 41.1\% & \ding{55} \\
                          &                       & 14B  & 0.38 & 31.1\% & \ding{55} \\
                          &                       & 72B  & 0.28 & 20.0\% & \ding{55} \\
\midrule
\multirow{4}{*}{MMLU-H}   & \multirow{4}{*}{CLS}  & 3B   & 0.35 & 27.8\% & \ding{55} \\
                          &                       & 7B   & 0.29 & 21.1\% & \ding{55} \\
                          &                       & 14B  & 0.22 & 13.3\% & \ding{55} \\
                          &                       & 72B  & 0.17 &  7.8\% & \ding{55} \\
\midrule
\multirow{4}{*}{JSON}     & \multirow{4}{*}{JSON} & 3B   & 0.06 &  0.0\% & \checkmark \\
                          &                       & 7B   & 0.11 &  1.1\% & \checkmark \\
                          &                       & 14B  & 0.00 &  0.0\% & \checkmark \\
                          &                       & 72B  & 0.00 &  0.0\% & \checkmark \\
\bottomrule
\end{tabular}
\caption{Minimum abstention lower bound ($\mathrm{LB} = \max\{0, (\mu-\alpha)/(1-\alpha)\}$) vs.~certification across model scales at $\alpha = 0.10$. $\mu$ values are from the original run; Cert?\ = whether any (score, bound) pair certifies at $\alpha = 0.10$. The 72B model's recomputed grid covers CoNLL, Few-NERD, WNUT, and TriviaQA only, so its NQ/MMLU/JSON rows reflect the original run. When $\mu \gg \alpha$, Proposition~\ref{prop:impossibility} of the main text implies any distribution-free certificate must abstain heavily; in our experiments, these settings do not certify. As scale increases, $\mu$ decreases across NER/QA/CLS yet remains above $\alpha$ even at 72B. Only JSON ($\mu \approx 0$) certifies at all scales. MMLU-Hum.\ at 72B ($\mu = 0.17$) nears the certification frontier.}
\label{tab:lb_scale}
\end{table}

\subsection{Miscalibration Analysis}

\begin{table}[t]
\centering
\small
\begin{tabular}{llrrr}
\toprule
Dataset & Task & ECE $\downarrow$ & Confidence & Risk \\
\midrule
CoNLL & NER  & 0.237 & 0.881 & 0.356 \\
FNERD & NER  & 0.418 & 0.878 & 0.540 \\
WNUT  & NER  & 0.332 & 0.850 & 0.480 \\
TQA   & QA   & 0.270 & 0.809 & 0.441 \\
NQ    & QA   & 0.611 & 0.806 & 0.804 \\
STEM  & CLS  & 0.343 & 0.824 & 0.516 \\
Hum.  & CLS  & 0.139 & 0.737 & 0.333 \\
JSON  & JSON & 0.048 & 0.967 & 0.033 \\
\bottomrule
\end{tabular}
\caption{LLM miscalibration on structured generation tasks. High ECE values show that model confidence is a poor proxy for correctness, motivating formal risk control.}
\label{tab:miscal}
\end{table}

\subsection{CRC Validity Across Risk Levels}

\begin{table}[t]
\centering
\scriptsize
\setlength{\tabcolsep}{3pt}
\begin{tabular}{llcccccccc}
\toprule
& & \multicolumn{3}{c}{NER} & \multicolumn{2}{c}{QA} & \multicolumn{2}{c}{CLS} & JSON \\
Model & $\alpha$ & CoNLL & FNERD & WNUT & TQA & NQ & STEM & Hum. & JSON \\
\midrule
\multirow{4}{*}{gemma3-4b} & 0.05 & --/1.00 & --/1.00 & --/1.00 & --/1.00 & --/1.00 & 1.00/0.99$^\dagger$ & 0.00/0.97 & 0.00/0.00 \\
& 0.10 & --/1.00 & --/1.00 & --/1.00 & --/1.00 & --/1.00 & 1.00/0.99$^\dagger$ & 0.33/0.96$^\dagger$ & 0.00/0.00 \\
& 0.15 & --/1.00 & --/1.00 & --/1.00 & 0.00/0.98 & --/1.00 & 1.00/0.99$^\dagger$ & 0.33/0.96$^\dagger$ & 0.00/0.00 \\
& 0.20 & --/1.00 & --/1.00 & --/1.00 & 0.00/0.97 & --/1.00 & 1.00/0.99$^\dagger$ & 0.33/0.96$^\dagger$ & 0.00/0.00 \\
\midrule
\multirow{4}{*}{qwen2.5-14b} & 0.05 & 0.01/1.00 & --/1.00 & --/1.00 & 0.00/0.99 & --/1.00 & --/1.00 & --/1.00 & 0.00/0.00 \\
& 0.10 & 0.07/0.99 & --/1.00 & --/1.00 & 0.00/0.99 & --/1.00 & --/1.00 & --/1.00 & 0.00/0.00 \\
& 0.15 & 0.07/0.98 & --/1.00 & --/1.00 & 0.12/0.35 & --/1.00 & --/1.00 & --/1.00 & 0.00/0.00 \\
& 0.20 & 0.15/0.97 & --/1.00 & --/1.00 & 0.16/0.27 & --/1.00 & --/1.00 & --/1.00 & 0.00/0.00 \\
\midrule
\multirow{4}{*}{qwen2.5-3b} & 0.05 & --/1.00 & --/1.00 & --/1.00 & --/1.00 & --/1.00 & --/1.00 & --/1.00 & 0.05/0.28 \\
& 0.10 & 0.08/0.93 & --/1.00 & --/1.00 & --/1.00 & --/1.00 & --/1.00 & --/1.00 & 0.08/0.00 \\
& 0.15 & --/1.00 & --/1.00 & --/1.00 & --/1.00 & --/1.00 & --/1.00 & --/1.00 & 0.08/0.00 \\
& 0.20 & 0.12/0.92 & --/1.00 & --/1.00 & 0.18/0.64 & --/1.00 & --/1.00 & --/1.00 & 0.08/0.00 \\
\midrule
\multirow{4}{*}{qwen2.5-7b} & 0.05 & --/1.00 & --/1.00 & --/1.00 & 0.00/0.99 & --/1.00 & --/1.00 & --/1.00 & 0.04/0.30 \\
& 0.10 & --/1.00 & --/1.00 & --/1.00 & 0.00/0.99 & --/1.00 & --/1.00 & --/1.00 & 0.08/0.17 \\
& 0.15 & --/1.00 & --/1.00 & --/1.00 & 0.00/0.99 & --/1.00 & --/1.00 & --/1.00 & 0.12/0.00 \\
& 0.20 & --/1.00 & --/1.00 & --/1.00 & 0.16/0.49 & --/1.00 & --/1.00 & 0.19/0.90 & 0.12/0.00 \\
\midrule
\multirow{4}{*}{ministral-8b} & 0.05 & 0.09/0.82$^\dagger$ & 0.20/0.92$^\dagger$ & 0.15/0.72$^\dagger$ & 0.05/0.63 & 0.20/1.00$^\dagger$ & 0.09/0.93$^\dagger$ & 0.07/0.62$^\dagger$ & 0.00/0.00 \\
& 0.10 & 0.13/0.69$^\dagger$ & 0.20/0.92$^\dagger$ & 0.15/0.72$^\dagger$ & 0.10/0.52 & 0.20/1.00$^\dagger$ & 0.19/0.90$^\dagger$ & 0.19/0.40$^\dagger$ & 0.00/0.00 \\
& 0.15 & 0.16/0.58$^\dagger$ & 0.20/0.92$^\dagger$ & 0.19/0.60$^\dagger$ & 0.15/0.37$^\dagger$ & 0.37/0.91$^\dagger$ & 0.19/0.90$^\dagger$ & 0.21/0.34$^\dagger$ & 0.00/0.00 \\
& 0.20 & 0.13/0.63 & 0.28/0.86$^\dagger$ & 0.22/0.51$^\dagger$ & 0.10/0.52 & 0.37/0.91$^\dagger$ & 0.19/0.90 & 0.13/0.53 & 0.00/0.00 \\
\bottomrule
\end{tabular}
\caption{CRC risk control across $\alpha$ levels (best score per dataset; original-run values). Format: test risk / abstention rate. -- indicates an empty emit set (full abstention); such degenerate settings are not counted as ``certified'' in our aggregate rates. $\dagger$: violation on test set. The 72B model is omitted (greedy-only scores; no per-cell sampling grid).}
\label{tab:crc_alpha}
\end{table}

\subsection{Cross-Task Controllability}
We define controllability $C(\alpha) = 1 - \mathrm{abstention}$ at $\alpha$. At $\alpha = 0.10$ with Qwen2.5-3B: JSON achieves $C = 0.996$, CLS $0.366$, QA $0.183$, NER $0.141$. The strong correlation with base risk confirms that CRC cannot compensate for weak models.

\subsection{Baselines Comparison}

\begin{table}[t]
\centering
\small
\begin{tabular}{lrrr}
\toprule
Method                         & Risk $\downarrow$ & Abstain $\downarrow$ & Viol.\ \% $\downarrow$ \\
\midrule
Always Answer                  & 0.434 & 0.000 & 90.0\% \\
Entropy Thr.                   & 0.021 & 0.868 & 12.5\% \\
TM Thr.\ (no guar.)            & 0.022 & 0.869 &  7.5\% \\
SC Thr.\ (no guar.)            & 0.019 & 0.846 &  7.5\% \\
CRC (Ours, guaranteed)         & \textbf{0.009} & \textbf{0.075} & \textbf{0.0\%} \\
\bottomrule
\end{tabular}
\caption{Aggregate comparison at $\alpha = 0.10$ (averaged over the five sampling-based models, eight datasets; the 72B model is greedy-only). Threshold methods tune threshold on calibration set without formal guarantees. CRC (guaranteed) reports only formally certified configurations.}
\label{tab:baselines}
\end{table}

\subsection{Data Efficiency}
Variance-aware bounds are substantially more data-efficient. We swept calibration fractions $\{10\%, 25\%, 40\%, 55\%, 70\%\}$ of each dataset over three split seeds (42--44), re-running the full 716-configuration grid at each fraction. Mean certification rates (Hoeffding / Bernstein / e-CRC): $1.1/1.9/4.2\%$ at 10\% (pooled counts 24/40/90 of 2{,}148); $4.1/5.5/8.0\%$ at 25\% (87/119/172); $6.1/8.3/10.5\%$ at 40\% (132/179/225); $7.2/9.8/11.4\%$ at 55\% (154/211/244); $7.5/10.5/11.9\%$ at 70\% (162/225/255). e-CRC's relative advantage over Bernstein is largest in the data-scarce regime (+125\% at 10\% calibration, where e-CRC also certifies $3.75\times$ Hoeffding) and narrows monotonically to +13\% at 70\%; e-CRC at 25\% calibration already exceeds Hoeffding at 70\%. \textbf{Pooling caveat:} the three seeds share the evaluation pool via different splits, so fraction-level comparisons indicate aggregate tendency rather than independent replications. The 10\% row is summarized in Table~\ref{tab:bound_compare} of the main text.

\section{Cross-Dataset ACI Transfer}
\label{app:aci_xfer}

\subsection{ACI Experimental Setup}
For the distribution shift experiments:
\begin{itemize}
\item \textbf{Temporal shift:} Split by index order (first 60\%/last 40\%). CoNLL-2003 articles are ordered by publication date; TriviaQA by source document; MMLU by subject cluster. Index order captures topic drift within corpora.
\item \textbf{Domain shift:} Calibrate on MMLU-STEM (400), test on MMLU-Humanities (400).
\item \textbf{Severity sweep:} Oversample high-risk examples ($R > \mathrm{median}(R)$) by factors $1.0$--$3.0\times$.
\item \textbf{Cross-dataset:} CoNLL$\to$WNUT (newswire$\to$social media), CoNLL$\to$FewNERD (newswire$\to$Wikipedia), TriviaQA$\to$NQ (trivia$\to$Google queries).
\end{itemize}
ACI uses $\gamma = 0.01$ with initial threshold from static CRC.

\begin{table}[t]
\centering
\small
\setlength{\tabcolsep}{2pt}
\begin{tabular}{llrrrrcc}
\toprule
Transfer & Model & $\mu_{\mathrm{tgt}}$ & Stat.\ $\hat{R}$ & ACI $\hat{R}$ & ACI abst. & Stat. & ACI \\
\midrule
\multirow{6}{*}{\shortstack[l]{CoNLL\\$\to$WNUT}}    & 3B   & 0.53 & 0.52 & 0.53 & 0.01 & \ding{55} & \ding{55} \\
                                   & 4B   & 0.49 & 0.00 & 0.48 & 0.03 & \checkmark$^*$ & \ding{55} \\
                                   & 7B   & 0.50 & 0.46 & 0.49 & 0.04 & \ding{55}  & \ding{55} \\
                                   & 8B   & 0.44 & 0.19 & 0.14 & 0.75 & \ding{55}  & \ding{55} \\
                                   & 14B  & 0.45 & 0.37 & 0.43 & 0.10 & \ding{55}  & \ding{55} \\
                                   & 72B  & 0.40 & 0.67 & 0.00 & 1.00 & \ding{55}  & \checkmark$^*$ \\
\midrule
\multirow{6}{*}{\shortstack[l]{CoNLL\\$\to$FewNERD}} & 3B   & 0.63 & 0.68 & 0.64 & 0.03 & \ding{55} & \ding{55} \\
                                   & 4B   & 0.55 & 0.00 & 0.54 & 0.03 & \checkmark$^*$ & \ding{55} \\
                                   & 7B   & 0.47 & 0.34 & 0.44 & 0.33 & \ding{55}  & \ding{55} \\
                                   & 8B   & 0.54 & 0.27 & 0.19 & 0.92 & \ding{55}  & \ding{55} \\
                                   & 14B  & 0.51 & 0.54 & 0.51 & 0.05 & \ding{55}  & \ding{55} \\
                                   & 72B  & 0.43 & 1.00 & 0.00 & 1.00 & \ding{55}  & \checkmark$^*$ \\
\midrule
\multirow{4}{*}{\shortstack[l]{TQA\\$\to$NQ}}        & 3B   & 0.84 & 0.56 & 0.77 & 0.51 & \ding{55}  & \ding{55} \\
                                   & 7B   & 0.81 & 0.52 & 0.79 & 0.17 & \ding{55}  & \ding{55} \\
                                   & 8B   & 0.79 & 0.39 & 0.34 & 0.91 & \ding{55}  & \ding{55} \\
                                   & 14B  & 0.75 & 0.44 & 0.73 & 0.08 & \ding{55}  & \ding{55} \\
\bottomrule
\end{tabular}
\caption{Cross-dataset ACI transfer at $\alpha = 0.10$ ($\gamma = 0.01$; $\hat{R}$ = risk among emitted outputs; values rounded to two decimals). Every target has base risk $\mu_{\mathrm{tgt}} > \alpha$. Static CRC violates on 14/16 runs and ACI also violates on 14/16; the exceptions ($^*$) all correspond to (near-)total abstention, i.e.\ (almost) nothing is emitted. Models: 3B/7B/14B/72B = Qwen2.5, 4B = Gemma-3, 8B = Ministral. Gemma-3-4B TQA$\to$NQ is skipped (no valid target pairs) and 72B TQA$\to$NQ lacks target data.}
\label{tab:aci_xfer}
\end{table}

\paragraph{Sensitivity to the ACI step size $\gamma$.} We swept $\gamma \in \{0.001, 0.002, 0.005, 0.01, 0.02, 0.05, 0.1\}$ across all 16 transfer runs. Under the emitted-risk criterion the violation count is identical at every $\gamma$: 14/16. Under the effective-risk criterion (emitted risk mass over all streamed examples, abstentions counted as zero risk), violations range from 9/16 ($\gamma \leq 0.002$) to 12/16 ($\gamma = 0.1$), while mean abstention falls from 59\% ($\gamma = 0.001$) to 26\% ($\gamma = 0.1$). Conclusions are therefore not sensitive to $\gamma$: no step size restores the emitted-risk target, because every transfer lands in the infeasible regime of Proposition~\ref{prop:impossibility} of the main text (per-model target-domain base risks: CoNLL$\to$WNUT 0.40--0.53, CoNLL$\to$FewNERD 0.43--0.63, TQA$\to$NQ 0.75--0.84; per-transfer static violation rates 5/6, 5/6, and 4/4). Mechanically, the telescoping identity of Theorem~\ref{thm:aci} pins the long-run \emph{effective} risk near $\alpha$ whenever the threshold clamp does not bind, which makes risk-among-emitted converge to roughly $\alpha$ divided by the emit rate, a quantity above $\alpha$ whenever abstention is nonzero; on hard transfers ACI therefore trades abstention for effective-risk control rather than restoring the per-emission guarantee, and on runs that saturate the clamp even the effective-risk identity no longer applies.

\section{Multi-Seed Robustness and Statistical Significance}
\label{app:multiseed}

To assess whether the certification ordering is an artifact of the seed-42 split, we reran the full 716-configuration grid under five split seeds (42--46). Certified counts are stable across seeds: Hoeffding 51/52/54/53/51, Bernstein 72/73/75/70/71, e-CRC 80/83/83/82/81. Certified-set violations per seed are 0/0/0 (seed 42), 0/0/0 (43), 1/1/1 (44), 0/0/0 (45), 0/0/0 (46) for Hoeffding/Bernstein/e-CRC respectively.

\paragraph{Paired significance analysis.} We pooled the $5 \times 716 = 3{,}580$ paired configuration observations and applied McNemar's exact test to each pair of methods (Table~\ref{tab:mcnemar}).

\begin{table}[t]
\centering
\small
\setlength{\tabcolsep}{3pt}
\begin{tabular}{@{}lrrrr@{}}
\toprule
Pair (A vs B) & Both & A-only & B-only & $p$-value \\
\midrule
Hoeffding vs Bernstein & 261 & 0 & 100 & $1.6 \times 10^{-30}$ \\
Bernstein vs e-CRC     & 361 & 0 & 48  & $7.1 \times 10^{-15}$ \\
Hoeffding vs e-CRC     & 261 & 0 & 148 & $5.6 \times 10^{-45}$ \\
\bottomrule
\end{tabular}
\caption{McNemar exact test on pooled paired observations ($5 \times 716 = 3{,}580$ per pair). Columns: both certify / A-only / B-only / exact two-sided $p$.}
\label{tab:mcnemar}
\end{table}

The headline finding is that \emph{zero} configurations are certified by Hoeffding but not Bernstein, and \emph{zero} by Bernstein but not e-CRC, across all 3{,}580 paired observations at $\alpha \in \{0.05, 0.10, 0.15, 0.20\}$: at strict targets the hierarchy holds without a single reversal. At relaxed targets ($\alpha \geq 0.25$) the Hoeffding--Bernstein ordering does reverse (Section~\ref{sec:relax}); Bernstein-not-e-CRC remains zero at every tested $\alpha$. \textbf{Pooling caveat.} Because the five seeds share the evaluation pool via different splits, the pooled $p$-values indicate aggregate consistency, not independent replications in the classical sense. The zero-reversal count is therefore the primary evidence; the $p$-values reinforce but do not replace it.

\paragraph{Pooled violation rates.} Across the five seeds, certified-set violation rates with Clopper--Pearson 95\% confidence intervals are: Hoeffding 1/261 ($0.38\%$, CI $[0.010\%, 2.12\%]$), Bernstein 1/361 ($0.28\%$, CI $[0.007\%, 1.53\%]$), and e-CRC 1/409 ($0.24\%$, CI $[0.006\%, 1.35\%]$); all upper limits are far below the $\delta = 0.10$ tolerance. All three pooled violations are the \emph{same} configuration: Ministral-8B, CoNLL, Self-Consistency, $\alpha = 0.2$, seed 44. For the Ministral-8B model specifically, pooled certified-set violations are 1/74 (Hoeffding), 1/100 (Bernstein), and 1/111 (e-CRC), with Clopper--Pearson 95\% upper limits 7.3\%, 5.4\%, and 4.9\%; Ministral is not an outlier beyond that single config.

\section{Threats to Validity}
\label{app:threats}

We report all 716 configurations; none excluded post hoc. Thresholds $\lambda^*$ determined entirely on calibration (Algorithm~\ref{alg:crc}); no hyperparameters tuned on test. ``Guaranteed'' means $\hat{R}_n^+(\lambda^*) \leq \alpha$ on calibration, where $\hat{R}_n^+(\lambda)$ denotes the upper confidence limit for the emit-set risk at $\lambda$ computed from the $n$ calibration points (the Hoeffding or Bernstein UCB, or the betting certificate's implied bound), implying $\mathbb{E}[R(\lambda^*)] \leq \alpha$ with probability $\geq 1 - \delta$ for i.i.d.\ calibration data. We fix $\delta = 0.1$, so CRC guarantees are probabilistic: rare test-set violations can occur even for certified configurations. Empirically, violations are near-zero among certified sets, suggesting the bounds are conservative in our settings. Threshold selection uses a pre-registered descending fixed-sequence scan with early stopping, so no Bonferroni correction is needed and no risk-monotonicity assumption is invoked (Lemma~\ref{lem:gridsearch})~\citep{angelopoulos2024ltt}; cross-bound comparisons are structural (nested sets at $\alpha \leq 0.20$), complemented by the McNemar analysis of \S\ref{app:multiseed}. Two resampling analyses quantify the stability of the certified counts. (i) \emph{Split-level re-randomization}: recomputing the full 716-configuration grid under $B = 200$ fresh 60/40 calibration/test splits (seeds 1001--1200) gives mean certified counts [95\% percentile intervals] of 53.1 [51, 56] for Hoeffding, 72.4 [68, 75] for Bernstein, and 82.4 [79, 86] for e-CRC. (ii) \emph{Config-level outcome bootstrap}: resampling the 716 seed-42 configuration outcomes 10{,}000 times gives 51 [38, 65], 72 [56, 88], and 80 [64, 96]. The two procedures answer different questions: (i) captures split-to-split variability of the pipeline, (ii) the sampling variability of the configuration population; the seed-42 point estimates (51/72/80) are consistent with both.

\section{Limitations and Future Work}
\label{app:lim}

\paragraph{Limitations.} Static guarantees are stated for i.i.d.\ calibration data; under shift they no longer apply, and ACI provides only the conditional effective-risk identity of Theorem~\ref{thm:aci}, not finite-sample control. Sampling-based scores require $K$ forward passes. Score fusion requires a disjoint split, reducing the effective calibration size. The MMLU classification datasets are small ($n = 400$ each); QA is evaluated on only two datasets, of which NQ Open is an outlier in base risk; the 72B model is evaluated on only two greedy-decoding scores. JSON-Extract uses template-based extraction tasks, and its near-ceiling certification rates may partly reflect template regularity. ACI conclusions are limited to the tested step-size range $\gamma \in [0.001, 0.1]$ (7 values). Evaluation spans 3B--72B; larger models (175B+) or closed-source APIs would further test generalizability.

\paragraph{Future work.} Learned scores optimized for risk--coverage; finite-sample ACI via e-processes; multi-task CRC with shared calibration; integration with constrained decoding; evaluation on 70B+ models.

\section{Additional Baselines and Failure Analysis}
\label{app:add_baselines}

\paragraph{Comparability to CP-Sets and conformal factuality.} CP-Sets~\citep{quach2024conformallm} focus on set-valued prediction (returning a set of candidate labels/answers) with coverage guarantees, whereas we certify single schema-validated structured outputs under task losses such as entity-level F1 and exact match. Bringing CP-Sets into our setting would require (i) a method to generate and score sets of structured candidates (e.g., multiple JSON/NER outputs), and (ii) a coverage definition tied to a structured distance (e.g., edit distance over entity sets). Conformal factuality~\citep{mohri2024conformalfactuality} certifies factual claims by composing retrieval and entailment-style checks; this is complementary to our focus on certifying format-valid structured generations with task-specific losses, and could be combined with our approach as an additional nonconformity signal.

\paragraph{Selective prediction baselines.} Learning-based rejection (e.g., SelectiveNet~\citep{geifman2019selectivenet}) can yield strong empirical coverage--risk tradeoffs, but typically does not provide distribution-free guarantees without an additional conformalization step; \citet{lee2024selectivegen} is a case in point of that added step, learning a rejection function on top of a frozen LLM that achieves a distribution-free FDR-control guarantee via a semi-supervised conformal procedure. Our CRC framework can wrap any score (including learned rejectors or retrieval-based confidence signals), provided exchangeability holds; thus CRC should be viewed as a guarantee layer rather than a competing uncertainty method.

\paragraph{Why ACI can still violate under shift.} ACI~\citep{gibbs2021aci} adapts thresholds online using realized feedback. While this helps under gradual shift, it cannot bypass feasibility: when the base risk $\mu$ is far above the target $\alpha$, any distribution-free method must abstain heavily (Proposition~\ref{prop:impossibility} of the main text). In our cross-dataset transfers, the remaining ACI violations tend to occur in such high-$\mu$ regimes or when feedback is delayed/noisy, which limits adaptation. Practically, we recommend using the feasibility check ($\mu$ vs.~$\alpha$) before deployment and treating ACI as a monitoring-and-adjustment tool when feedback is available.

\paragraph{Score brittleness.} Nonconformity scores can in principle be sensitive to tokenizer/model family and quantization artifacts (e.g., AWQ). In our recomputed runs, certified-set violations were too rare to support any pattern-level claim (a single pooled configuration across five seeds; \S\ref{app:multiseed}), so we retain this as a deployment consideration rather than an empirical finding: use score functions that are stable across decoding/implementation variations and validate guarantees per deployment stack.

\paragraph{Datasets.} Table~\ref{tab:datasets_supp} summarizes the eight evaluation datasets (Section~\ref{sec:setup}).

\begin{table}[t]
\centering
\small
\begin{tabular}{llrl}
\toprule
Dataset & Task & Size & Domain \\
\midrule
CoNLL-2003   & NER  & 3{,}453 & News \\
Few-NERD     & NER  & 2{,}000 & Wikipedia \\
WNUT-17      & NER  & 1{,}287 & Social media \\
TriviaQA     & QA   & 1{,}000 & Trivia \\
NQ Open      & QA   & 3{,}610 & Wikipedia \\
MMLU-STEM    & CLS  & 400     & STEM \\
MMLU-Hum.    & CLS  & 400     & Humanities \\
JSON-Extract & JSON & 600     & Structured \\
\bottomrule
\end{tabular}
\caption{Datasets used in our evaluation.}
\label{tab:datasets_supp}
\end{table}

\begin{table}[t]
\centering
\small
\setlength{\tabcolsep}{2.5pt}
\begin{tabular}{lccccc}
\toprule
& Struct. & Risk & Lower & Bound & Shift \\
Line of work & LLM & abst. & bound & hier. & adapt. \\
\midrule
CP sets                       & \ding{55} & \ding{55} & \ding{55} & \ding{55} & \ding{55} \\
CRC / LTT risk control        & \ding{55} & \checkmark & \ding{55} & \ding{55} & \ding{55} \\
Selective pred. / reject opt. & \ding{55} & \checkmark & \checkmark & \ding{55} & \ding{55} \\
ACI                           & \ding{55} & \ding{55} & \ding{55} & \ding{55} & \checkmark \\
Betting / e-values seq.       & \ding{55} & \ding{55} & \ding{55} & \ding{55} & \ding{55} \\
\midrule
This work                     & \checkmark & \checkmark & \checkmark & \checkmark & \checkmark \\
\bottomrule
\end{tabular}
\caption{Positioning relative to closely related lines of work (Section~\ref{sec:intro}).}
\label{tab:positioning_supp}
\end{table}

\section{Sharpened Abstention Floor: Attainability, Extremal Law, and Decomposition}
\label{app:dirA}

This section supplies the attainability, extremal-law, minimax, and decomposition results supporting Proposition~\ref{prop:impossibility} and Remark~\ref{rem:decomp} of the main text; the lower bound was proved in \S\ref{app:proofs}. Throughout, $S \in [0,1]$ is the observable score, $R \in [0,1]$ the risk with $\mu = \mathbb{E}[R]$ and $M = \operatorname*{ess\,sup} R$; a selective rule is $\phi: \mathcal{X} \to [0,1]$ with coverage $c(\phi) = \mathbb{E}[\phi]$, emitted risk $\mathrm{er}(\phi) = \mathbb{E}[R\phi]/\mathbb{E}[\phi]$, abstention $a(\phi) = 1 - c(\phi)$, and abstention target $A_{\mathrm{DF}}(P;\alpha) = \inf\{a(\phi): \phi \text{ is } \sigma(S)\text{-measurable and } \alpha\text{-valid}\}$. Randomized thresholds split atoms of $S$.

\paragraph{Relation to prior work.} The per-fixed-law achievability below---the optimal selective rule thresholds the calibrated conditional risk, and the risk--coverage curve has a rearrangement lower envelope---is classical \emph{known-distribution} reject-option theory: \citet{chow1970optimum} (threshold on the posterior), \citet{elyaniv2010selective} (risk--coverage), and especially \citet{franc2023reject} (bounded-abstention and bounded-improvement optima). We specialize these to a scalar score. Our contribution is the distribution-free framing over mean-constrained laws: the $\operatorname{ess\,sup}$ sharpening, the extremal-law condition, the minimax degeneracy, and the gap decomposition---\emph{not} the achievability, which we do not claim as new.

\begin{proposition}[Attainability, extremal law, and minimax degeneracy]
\label{prop:s-extremal}
Fix $\mu > \alpha$ and $M \leq 1$.
\emph{(i) Attainability.} There is a law $P^\star$ of $(S,R)$ with $\mathbb{E}[R] = \mu$, $\operatorname{ess\,sup} R = M$, and a randomized score-threshold rule $\phi^\star$ (score-measurable) that is $\alpha$-valid with $a(\phi^\star) = (\mu-\alpha)/(M-\alpha)$. Hence $(\mu-\alpha)/(M-\alpha)$ is the largest constant lower-bounding abstention across all mean-$\mu$, $\operatorname{ess\,sup}$-$M$ laws, and the outer infimum is attained (a minimum).
\emph{(ii) Extremal condition.} For a general law, the floor is attained by some $\alpha$-valid rule iff $\mathbb{P}(R = M) \geq (\mu-\alpha)/(M-\alpha)$.
\emph{(iii) Minimax degeneracy.} $\inf_{P:\,\mathbb{E}[R]=\mu} A_{\mathrm{DF}}(P;\alpha) = (\mu-\alpha)/(M-\alpha)$ while $\sup_{P:\,\mathbb{E}[R]=\mu} A_{\mathrm{DF}}(P;\alpha) = 1$ (attained by any law with $S \perp R$). There is no minimax saddle over the unrestricted mean-$\mu$ class; the floor is a tight feasibility frontier, not a saddle value.
\end{proposition}

\begin{proof}
(i) Take $R \in \{0, M\}$ with $\mathbb{P}(R=M) = \mu/M$ (mean $\mu$; $\mu \leq M$) and a perfectly ordering score $S = \mathbb{1}\{R=0\}$. Emit every $S=1$ point and each $S=0$ point with probability $q$: $\phi^\star = \mathbb{1}\{S=1\} + q\,\mathbb{1}\{S=0\}$. Then $\mathbb{E}[R\phi^\star] = q\mu$ and $c(\phi^\star) = (1-\mu/M) + q\mu/M$; choosing $q = \alpha(M-\mu)/(\mu(M-\alpha)) \in [0,1]$ makes $\mathrm{er}(\phi^\star) = \alpha$ and $a(\phi^\star) = (\mu/M)(1-q) = (\mu-\alpha)/(M-\alpha)$ (verified numerically for $M \in \{0.6, 0.8, 1\}$). The rule is $\sigma(S)$-measurable; the randomization is essential since $S$ is two-valued. The Proposition~\ref{prop:impossibility} lower bound gives ``$\geq$'', so the value is exactly the floor and is attained.
(ii) \emph{Sufficiency:} with $x = (\mu-\alpha)/(M-\alpha)$, if $\mathbb{P}(R=M) \geq x$, abstain on a measurable subset of $\{R=M\}$ of mass $x$ and emit the rest; then emitted risk $= (\mu - Mx)/(1-x) = \alpha$ and $a = x$. \emph{Necessity:} equality throughout the Proposition~\ref{prop:impossibility} chain forces $R = M$ a.s.\ on the abstain set, so a mass-$x$ abstain set lies in $\{R=M\}$, whence $\mathbb{P}(R=M) \geq x$.
(iii) The infimum is (i). For the supremum, if $S \perp R$ then $\mathbb{E}[R \mid S \geq \lambda] = \mu > \alpha$ for every non-empty emit set, so validity forces $c = 0$, i.e.\ $A_{\mathrm{DF}} = 1$; and $A_{\mathrm{DF}} \leq 1$ always.
\end{proof}

\begin{proposition}[Two-gap decomposition (Remark~\ref{rem:decomp})]
\label{prop:s-decomp}
For any law of $(S,R)$ with $\mu > \alpha$, let $g(s) = \mathbb{E}[R \mid S=s]$ and let $q_u$ be the quantile function of $R$; by the Hardy--Littlewood rearrangement inequality, $\rho^\dagger(c) = \tfrac1c \int_0^c q_u\,du$ is the pointwise-minimal risk--coverage curve. With $A^\dagger(R;\alpha) = 1 - \sup\{c: \rho^\dagger(c) \leq \alpha\}$,
\[
A_{\mathrm{DF}}(P;\alpha) = \frac{\mu-\alpha}{M-\alpha} + \underbrace{\Big[A^\dagger - \tfrac{\mu-\alpha}{M-\alpha}\Big]}_{\Delta_{\mathrm{disp}}(R)\,\geq\,0} + \underbrace{\big[A_{\mathrm{DF}} - A^\dagger\big]}_{\Delta_{\mathrm{ord}}(S)\,\geq\,0},
\]
where $\Delta_{\mathrm{disp}} = 0$ iff $\mathbb{P}(R=M) \geq (\mu-\alpha)/(M-\alpha)$ and $\Delta_{\mathrm{ord}} = 0$ iff the optimal emit set $\{g(S) \leq \tau\}$ coincides ($P$-a.s.) with a bottom-risk set. Among $\sigma(S)$-measurable $\alpha$-valid rules, abstention is minimized by a threshold on the calibrated score $g(S)$.
\end{proposition}

\begin{proof}
By the tower property $\mathbb{E}[R\phi] = \mathbb{E}[g(S)\phi]$ for $\sigma(S)$-measurable $\phi$, so maximizing $c = \mathbb{E}[\phi]$ subject to $\mathbb{E}[(g(S)-\alpha)\phi] \leq 0$, $0 \leq \phi \leq 1$, is a fractional-knapsack (Neyman--Pearson) problem whose optimum thresholds $g(S)$ with the constraint binding, giving $A_{\mathrm{DF}} = 1 - \mathbb{E}[\phi^\ast]$. The decomposition telescopes. $\Delta_{\mathrm{ord}} \geq 0$: $A_{\mathrm{DF}}$ optimizes over $\sigma(S) \subseteq \sigma(X,R)$ and $\rho^\dagger$ is the pointwise-minimal risk--coverage curve, so $A_{\mathrm{DF}} \geq A^\dagger$, with equality iff the $S$-optimal emit set is a bottom-$R$ set. $\Delta_{\mathrm{disp}} \geq 0$: apply Proposition~\ref{prop:impossibility} to the oracle rule ($A^\dagger \geq (\mu-\alpha)/(M-\alpha)$); the equality condition is Proposition~\ref{prop:s-extremal}(ii). Computed checks: two-point $\{0,1\}$ with $\mu=.5, \alpha=.25$ gives $A^\dagger = 1/3 = $ floor ($\Delta_{\mathrm{disp}}=0$); uniform $R$ on $[0,1]$ gives $A^\dagger = .5 > 1/3$ ($\Delta_{\mathrm{disp}} = 1/6$); $S \perp R$ gives $\Delta_{\mathrm{ord}} = 1 - A^\dagger$.
\end{proof}

\paragraph{Bridge to fusion.} Calibrated fusion $g = \sigma(w^\top \mathbf{s} + b)$ cannot change $\Delta_{\mathrm{disp}}$ (fixed by the risk marginal and $M$) and can only act on $\Delta_{\mathrm{ord}}$ by making $g(S)$ a better monotone proxy for $R$; this links score quality to certification \emph{in principle}. Empirically the link is weak on our grid: a disjoint-split re-evaluation (\S\ref{app:round2}) finds fusion's gain is on AUROC (ranking), not on certified-set size.

\section{Certification Phase Diagram: Reversal Frontier and Pareto Necessity (Full Proofs)}
\label{app:dirC}

Notation as in Proposition~\ref{prop:ordering}: $L = \log(2/\delta)$, $w_H(n) = \sqrt{L/(2n)}$, $b(n) = 7L/(3(n-1))$, $w_B(n,\hat\sigma^2) = \sqrt{2\hat\sigma^2 L/n} + b(n)$, $A(n) = w_H(n) - b(n)$, and $c(n,\delta) = \frac{n}{2L}A(n)^2$ (valid when $A(n) > 0$; $n \geq 35$ at $\delta = 0.1$). The feasibility coupling is Bhatia--Davis: for risks in $[0,1]$, $\hat\sigma^2 \leq \hat{R}(1-\hat{R})$, with \emph{equality} for $0/1$ losses (QA exact-match, CLS accuracy). The pairwise width comparison $U_B \leq U_H \iff \hat\sigma^2 \leq c(n,\delta)$ is Proposition~\ref{prop:ordering} and is known~\citep{maurer2009empirical}; the new content is the certification-level frontier and the Pareto structure. All numbers are from \texttt{theory/verify\_phasediagram.py} (seed-42 grid, the paper's own certification code), output \texttt{verify\_phasediagram\_out.json}.

\begin{proposition}[Reversal frontier (Proposition~\ref{prop:phase}(i))]
\label{prop:s-reversal}
Fix $n, \delta$ with $A(n) > 0$. A feasible emit set of size $n$ that Hoeffding certifies but empirical Bernstein does not (i.e.\ $U_H \leq \alpha < U_B$ with $0 \leq \hat\sigma^2 \leq \hat{R}(1-\hat{R})$) exists iff $(\alpha - w_H)(1 - \alpha + w_H) > c(n,\delta)$, equivalently $\alpha > \alpha^*(n,\delta) = w_H(n) + \tfrac12(1 - \sqrt{1-4c(n,\delta)}\,)$. For $0/1$ losses the Bhatia--Davis bound is an equality, so $\alpha^*$ is the exact reversal boundary at $\hat{R} = \alpha - w_H$.
\end{proposition}

\begin{proof}
A reversal at mean $\hat{R}$ needs (a) $\hat{R} \leq \alpha - w_H$ (Hoeffding certifies) and (b) $\hat{R} + w_B > \alpha$, i.e.\ $\hat\sigma^2 > \sigma^2_{\mathrm{crit}}(\hat{R}) := \frac{n}{2L}(\alpha - \hat{R} - b)^2$ when $\alpha - \hat{R} - b \geq 0$. Since $w_B$ increases in $\hat\sigma^2$, the easiest witness sets $\hat\sigma^2$ at its feasible maximum $\hat{R}(1-\hat{R})$. So a reversal exists iff $\exists\, \hat{R} \in [0, \alpha - w_H]$ with $h(\hat{R}) := \hat{R}(1-\hat{R}) - \frac{n}{2L}(\alpha - \hat{R} - b)^2 > 0$. On $[0, \alpha - w_H]$ (with $\alpha - w_H \leq \tfrac12$), $h'(\hat{R}) = (1 - 2\hat{R}) + \frac{n}{L}(\alpha - \hat{R} - b) \geq 0$ (both terms $\geq 0$, the second because at the endpoint $\alpha - \hat{R} - b = A \geq 0$), so $h$ is increasing and maximized at $\hat{R}^* = \alpha - w_H$, where $\alpha - \hat{R}^* - b = A$ and $\sigma^2_{\mathrm{crit}}(\hat{R}^*) = \frac{n}{2L}A^2 = c$. Thus $h(\hat{R}^*) > 0 \iff \hat{R}^*(1 - \hat{R}^*) > c \iff (\alpha - w_H)(1 - \alpha + w_H) > c$; solving $\hat{R}^*(1-\hat{R}^*) = c$ on the branch $\hat{R}^* \leq \tfrac12$ gives $\alpha > \alpha^*$. For $0/1$ losses no Bhatia--Davis slack is discarded, so ``if'' becomes exact.
\end{proof}

\paragraph{Two mechanisms and numerical confirmation.} For $n < 35$ ($A(n) \leq 0$ at $\delta = 0.1$), $w_B \geq w_H$ for every $\hat\sigma^2 \geq 0$ (the additive $b(n)$ alone exceeds $w_H$), so Bernstein is universally looser and reversals occur even at $\hat\sigma^2 < c$; the frontier $\alpha^*$ is claimed only in the $A(n) > 0$ regime. On the seed-42 grid (28{,}761 candidate emit sets $\times$ 7 targets; 1{,}091 pointwise reversals): in the $A > 0$ regime, \textbf{0 of 947} reversals violate $\hat\sigma^2 > c$ and \textbf{0} occur below $\alpha^*(n)$; the $A \leq 0$ regime contributes 144 reversals, 27 of them with $\hat\sigma^2 < c$ (the additive mechanism). $\alpha^*$ is smooth with minimum $0.169$ at $n = 88$, rising to $0.20$ ($n{=}300$), $0.23$ ($n{=}600$), $0.32$ ($n{=}5000$), $\to \tfrac12$ as $n \to \infty$; $c(120) = 0.056$, $c(600) = 0.147$ match Proposition~\ref{prop:ordering}. The Bhatia--Davis ratio $\hat\sigma^2/[\hat{R}(1-\hat{R})]$ has medians QA $1.00$, CLS $1.00$, NER $0.80$, JSON $0.27$, confirming that binary losses sit on the frontier. Config-level reversals (Hoeffding certifies, Bernstein does not) first appear at $\alpha = 0.25$ (H $= 31 >$ B $= 30$), consistent with $\alpha^*$ being necessary but realized only where a real emit set attains $\hat{R} \approx \alpha - w_H$ with near-maximal variance.

\begin{proposition}[Pareto necessity (Proposition~\ref{prop:phase}(ii))]
\label{prop:s-pareto}
Over $(n, \hat{R}, \hat\sigma^2, \alpha)$ with $\alpha \in \{0.05, \dots, 0.20, 0.25, 0.30, 0.40\}$: (i) no single bound dominates the other two; (ii) e-CRC statistically subsumes empirical Bernstein~\citep{maurer2009empirical,waudbysmith2024betting} (0 configurations certified by Bernstein but not e-CRC); (iii) the best-of-three certified set equals $\Lambda_H \cup \Lambda_{\mathrm{e\text{-}CRC}}$ and strictly dominates every single bound, so $\{\text{Hoeffding}, \text{e-CRC}\}$ is minimal sufficient and Bernstein is a closed-form convenience.
\end{proposition}

\begin{proof}[Proof (constructive, seed-42 grid)]
(i) $H$-only $= 3$ (e.g.\ Qwen2.5-3B/TriviaQA/self-consistency/$\alpha = 0.30$: selected emit set $n = 222$, $\hat{R} = 0.216$, binary EM so $\hat\sigma^2 = 0.170$ on the frontier, $U_H = 0.298 \leq 0.30 < U_B = 0.316$, and $\alpha^*(222) = 0.188 < 0.30$; betting also fails), and e-CRC-only $= 10$ (e.g.\ JSON at $\alpha = 0.05$--$0.10$, data-abundant and low-variance, where the one-sided $\log 2$ threshold saving and the absent additive term let betting certify while both UCBs miss). (ii) Verified $B$-not-e-CRC $= 0$ across the grid including $\alpha \in \{0.25, 0.30, 0.40\}$, hence $\Lambda_B \subseteq \Lambda_{\mathrm{e\text{-}CRC}}$ and $B$-only $= 0$. (iii) Summed over the 7 targets, best-of-three $= 212 = $ best-of-$\{H, \text{e-CRC}\}$, while e-CRC-alone $= 209$ and $H$-alone $= 181$; both $H$ and e-CRC are Pareto-necessary.
\end{proof}

A second-order growth-rate argument (Taylor expansion of $\log$-wealth, not an exact boundary) explains e-CRC's data-scarce edge: empirically at 10\% calibration, certifications are H $= 7$, B $= 7$, e-CRC $= 19$, matching the regime where the $\log 2$ saving and absent additive term dominate.

\section{Hoeffding--Bentkus Domination (Analysis Remark)}
\label{app:dirD}

This supports the Discussion remark that the additive-Bernstein reversal is an artifact and that the inverted Learn-Then-Test $p$-value dominates Hoeffding unconditionally. \textbf{This is an analysis remark, not a change of method: the paper's reported certificates remain Hoeffding, empirical Bernstein, and e-CRC, with their zero-violation record (0/51, 0/72, 0/80 at seed 42) intact.} The Hoeffding--Bentkus (HB) certificate is \emph{not a new primitive}: it is the standard LTT $p$-value~\citep{angelopoulos2024ltt}, combining the Chernoff--Hoeffding (KL) tail with the Bentkus binomial-tail bound~\citep{bentkus2004hoeffding}; we claim no novelty for it.

For $H_0: \mu \geq \alpha$ and observed mean $\hat{R} < \alpha$, define $p_{\mathrm{CH}} = \exp(-n\,\mathrm{kl}(\hat{R}\|\alpha))$ ($\mathrm{kl}$ the binary relative entropy) and $p_{\mathrm{Bent}} = e\cdot\mathbb{P}(\mathrm{Binom}(n,\alpha) \leq \lceil n\hat{R}\rceil)$, and certify $\lambda$ iff $p_{\mathrm{HB}} := \min(p_{\mathrm{CH}}, p_{\mathrm{Bent}}) \leq \delta$. Both are nonincreasing in $\hat{R}$, so their rejection regions are nested lower intervals and the $\min$ is a valid level-$\delta$ $p$-value; validity of the selected threshold follows from the same descending fixed-sequence argument as Lemma~\ref{lem:gridsearch}.

\begin{proposition}[Unconditional Hoeffding domination]
\label{prop:s-hb}
For every $n$, $\hat\sigma^2$, and $\alpha$: if additive Hoeffding certifies $\lambda$ (i.e.\ $\hat{R} + \sqrt{L/(2n)} \leq \alpha$), then HB certifies $\lambda$. There is no variance condition and no reversal.
\end{proposition}

\begin{proof}
Additive Hoeffding certifies iff $\alpha - \hat{R} \geq \sqrt{L/(2n)}$. Pinsker's inequality $\mathrm{kl}(\hat{R}\|\alpha) \geq 2(\alpha - \hat{R})^2 \geq 2 \cdot L/(2n) = L/n$ gives $p_{\mathrm{CH}} \leq e^{-L} = \delta/2 < \delta$, so $p_{\mathrm{HB}} \leq p_{\mathrm{CH}} \leq \delta$. The $\delta/2$ slack comes from the two-sided constant $L = \log(2/\delta)$ used in the additive Hoeffding bound.
\end{proof}

\paragraph{Validity versus tightness.} Being valid but tighter, HB operates nearer the $\delta$ budget rather than the additive bounds' near-zero-violation regime. On the seed-42 strict grid it certifies 94 versus 51/72/80 with $\Lambda_{\mathrm{HB}} \supseteq \Lambda_{\mathrm{Hoeff}}, \Lambda_{\mathrm{Bern}}, \Lambda_{\mathrm{e\text{-}CRC}}$; pooled over seeds 42--46 its certified-set violation rate is $6/479 = 1.25\%$ (Clopper--Pearson 95\% upper $2.46\% \ll \delta = 0.1$), and an adversarial boundary Monte-Carlo (400{,}000 draws at $\mu = \alpha$) gives worst-case false-certification $7.7\% < \delta$. HB and empirical Bernstein are Pareto-incomparable in general (a near-zero-variance corner with $\hat{R}$ near $\alpha$ lets Bernstein certify where HB does not), but HB dominates Bernstein on all binary-risk and all LLM emit sets in our grid. We report the additive/e-CRC certificates in the main text to preserve the zero-violation narrative; HB is noted only as the unconditionally-Hoeffding-dominating valid alternative under this validity-versus-tightness trade-off. Scripts: \texttt{recompute/dirD\_*.py}.

\section{A Constructive Anytime-Valid Shift Monitor (Full Construction)}
\label{app:dirE}

This section gives the construction, validity proofs, and numbers behind the constructive-counterpart subsection of the main text and Proposition~\ref{prop:obstruction}. \textbf{The machinery is known and we claim no novelty for it}: time-uniform confidence sequences~\citep{howard2021time}, betting/e-processes~\citep{waudbysmith2024betting,ramdas2023safe}, and the closely related anytime-valid selective-risk methods A-RCPS~\citep{xu2024arcps}, Conformal Selective Acting~\citep{khosravi2026csa}, and anytime-valid CRC~\citep{hultberg2026anytime}. Our contribution is the demonstrated constructive dual of Proposition~\ref{prop:impossibility} on the paper's \emph{own} failed transfers, plus the emit-only obstruction (Proposition~\ref{prop:obstruction}).

\paragraph{Setup and filtration.} At step $t$, nature reveals $x_t$, the model produces $\hat{y}_t$ and score $s_t$, we choose a predictable threshold $\lambda_t$ ($\mathcal{F}_{t-1}$-measurable), emit iff $E_t = \mathbb{1}\{s_t \geq \lambda_t\}$, and observe $r_t = R(\hat{y}_t, y_t) \in [0,1]$. With $\mathcal{G}_t = \mathcal{F}_{t-1} \vee \sigma(x_t, s_t)$, the conditional emitted risk is $q_t = \mathbb{E}[r_t \mid \mathcal{G}_t]$; we assume nothing about $\{q_t\}$ (arbitrary, adversarial drift). \emph{Full feedback} observes $r_t$ at every $t$ (a verifier, or labels on abstained inputs too); \emph{emit-only feedback} observes $r_t$ only when $E_t = 1$.

\begin{theorem}[Per-threshold anytime confidence sequence]
\label{thm:s-e1}
Fix $\lambda$ and let $I_t = \{u \leq t: s_u \geq \lambda\}$, $n_t = |I_t|$, $A_t = \frac{1}{n_t}\sum_{u \in I_t} r_u$, $M_t = \frac{1}{n_t}\sum_{u \in I_t} q_u$. Under full feedback, with $U_t = A_t + B_\delta(n_t)/n_t$ and $B_\delta(n) = \sqrt{(n/4 + \rho)(2\log(1/\delta) + \log((n/4+\rho)/\rho))}$,
\[
\mathbb{P}(\exists t: M_t > U_t) \leq \delta
\]
for an arbitrary (non-exchangeable, adversarially drifting) risk process.
\end{theorem}

\begin{proof}
$S_t = \sum_{u \in I_t}(r_u - q_u)$ is a martingale with increments in $[-1,1]$, since $\mathbb{1}\{s_u \geq \lambda\}$ is $\mathcal{G}_u$-measurable and $\mathbb{E}[r_u - q_u \mid \mathcal{G}_u] = 0$. By Hoeffding's lemma, $L_t(\eta) = \exp(\eta S_t - (\eta^2/8)n_t)$ is a nonnegative supermartingale with $L_0 = 1$; mixing over $\eta \sim N(0, 1/\rho)$ and applying Ville's inequality~\citep{howard2021time} gives $\mathbb{P}(\exists t: |S_t| \geq B_\delta(n_t)) \leq \delta$. The event $M_t > U_t$ is $-S_t > B_\delta(n_t)$, whose probability is $\leq \delta$ (lower tail of the symmetric mixture).
\end{proof}

The monitor SAVeR-FF maintains $U_t(\lambda)$ over a predictable grid at Bonferroni level $\delta/G$ and emits at $\lambda_t = \min\{\lambda: U_{t-1}(\lambda) \leq \alpha\}$; every emission then occurs at a threshold certified $\leq \alpha$ anytime. Its abstention obeys the floor of Proposition~\ref{prop:impossibility}, and when a fixed $\lambda^*$ stays below $\alpha$ its emit rate $\to \mathbb{P}(s \geq \lambda^*)$: abstention tracks the floor when the score stays informative under shift and saturates to 1 when the stream is genuinely infeasible.

\begin{theorem}[Self-selected supermartingale validity]
\label{thm:s-e2}
Under emit-only feedback, the emitted-substream betting process $W_t = \prod_{u \leq t}(1 + E_u \kappa_u (r_u - \alpha))$ with predictable $E_u$ and $\kappa_u \in [0, 1/\alpha]$ is, under $H_0: q_u \leq \alpha$ for every emitted $u$, a nonnegative $\mathcal{F}_t$-supermartingale, so $\mathbb{P}(\exists t: W_t \geq 1/\delta) \leq \delta$.
\end{theorem}
\begin{proof}
$\mathbb{E}[W_t \mid \mathcal{G}_u] = W_{t-1}(1 + E_u\kappa_u(q_u - \alpha))$: on $\{E_u = 1\}$, $q_u - \alpha \leq 0$ and $\kappa_u \geq 0$ so the factor is $\leq 1$; on $\{E_u = 0\}$ it is $1$; nonnegativity from $\kappa_u \leq 1/\alpha$. The predictable multiplier $E_u\kappa_u$ makes self-selection harmless---abstained steps contribute factor 1 and no $r_u$ is needed there---so the feedback required coincides with the feedback obtained.
\end{proof}

\begin{proposition}[Emit-only feedback obstruction (Proposition~\ref{prop:obstruction})]
\label{prop:s-e3}
Under emit-only feedback and arbitrary shift, any procedure guaranteeing $\mathbb{P}(\forall t: A_t \leq \alpha) \geq 1 - \delta$ ($\delta < 1$) that emits with positive probability can be forced to violate.
\end{proposition}
\begin{proof}
The emit decision at $t$ is $\mathcal{G}_t$-measurable and cannot depend on $r_t$. Let the adversary set $r_t = 1$ on the first step the procedure emits: then $A = 1 > \alpha$, with probability equal to the procedure's probability of ever emitting. To keep the guarantee the procedure must emit nothing.
\end{proof}

ACI updates on $\mathrm{eff}_t = r_t\,\mathbb{1}\{\mathrm{emit}_t\}$ (emit-only feedback), so it can learn a region's risk only by emitting into it; Proposition~\ref{prop:s-e3} shows this exploration cost is charged as violations under shift---a model-level reason, complementing the feasibility-level Proposition~\ref{prop:impossibility}, that ACI cannot restore the emitted target. The escape is full feedback.

\section{Additional Analyses (Audited $M$, Disjoint-Split Fusion, Robustness)}
\label{app:round2}

This section supplies four supporting analyses recomputed from the cached per-example (score, risk) data with the paper's own certification library (no new model inference): the audited essential supremum $M$ behind Proposition~\ref{prop:impossibility}, a disjoint-split re-evaluation of score fusion, the $\delta$-sensitivity of the certified counts, and per-task and multi-seed breakdowns of the phase diagram. Figure~\ref{fig:phase} first visualizes the paper's two headline results.

\begin{figure}[t]
\centering
\includegraphics[width=0.94\textwidth]{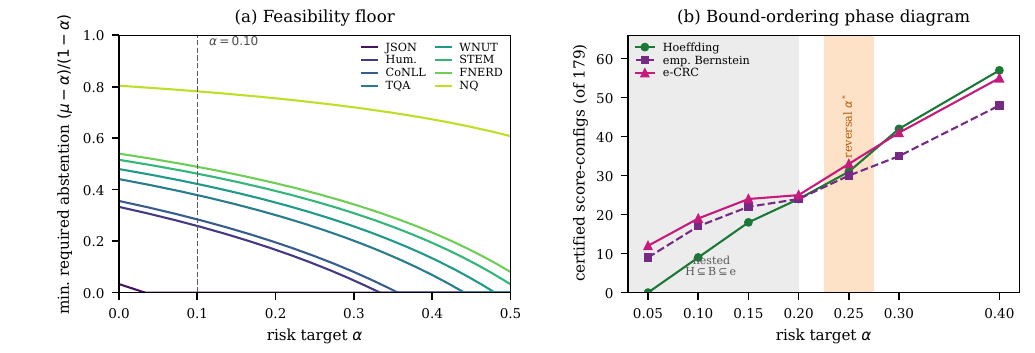}
\caption{The two headline results of the main text, visualized. \emph{(a)} Minimum required abstention $(\mu-\alpha)/(1-\alpha)$ vs.\ risk target $\alpha$, one curve per dataset at its measured base risk $\mu$ (main-text Table~\ref{tab:base_risk}). A curve reaches zero exactly at $\alpha=\mu$; every point to its left is provably infeasible without abstaining at least that much (main-text Proposition~\ref{prop:impossibility}). At the headline target $\alpha=0.10$ (dashed), seven of eight datasets sit strictly above zero abstention, which is why strict-target certification is rare. \emph{(b)} Number of certified score-configurations (of 179) for each bound as $\alpha$ relaxes. Below $\alpha\approx0.20$ (grey) the three certified sets are strictly nested $\Lambda^*_{\mathrm{Hoeff}}\subseteq\Lambda^*_{\mathrm{Bern}}\subseteq\Lambda^*_{\mathrm{e\text{-}CRC}}$ (main-text Proposition~\ref{prop:ordering}); past the reversal band (orange) Hoeffding overtakes empirical Bernstein, the empirical signature of the closed-form frontier $\alpha^*(n,\delta)$ (main-text Proposition~\ref{prop:phase}).}
\label{fig:phase}
\end{figure}

\subsection{Audited Essential Supremum}
\label{app:audit_M}

Proposition~\ref{prop:impossibility}'s sharpened floor $(\mu-\alpha)/(M-\alpha)$ improves on the conservative $(\mu-\alpha)/(1-\alpha)$ only when $M = \operatorname{ess\,sup} R < 1$. We audit $M$ the way the proposition says it should be deployed (the empirical sample maximum of the per-example risk, pooled over calibration and test, averaged over the five sampling-based models). Table~\ref{tab:audit_M} reports the result: on every NER/QA/CLS dataset the audited $M = 1.00$ \emph{exactly}, because a complete per-example miss (entity-$F_1 = 0$, or a wrong exact-match/label) occurs in every model$\times$dataset cell (on $18$--$53\%$ of examples). The sharpened floor therefore coincides numerically with the conservative $M{=}1$ floor for all NER/QA/CLS rows, and moves no tabulated $\mathrm{LB}$ value in this paper. Only JSON field-$F_1$ has $M<1$ (audited $\hat{M} \leq 0.33$), but there $\mu < \alpha$ already, so the floor is $0$ under either formula. The sharpening is a real, provable generalization (it would bind on losses whose per-example value cannot reach $1$, e.g.\ macro-$F_1$, $1-\mathrm{BLEU}$, $1-\mathrm{ROUGE}$), but it does not tighten any number on our grid.

\begin{table}[t]
\centering
\small
\begin{tabular}{llrrr}
\toprule
Dataset & Task & $\bar\mu$ & $\hat{M}$ & floor gap @ $\alpha{=}0.1$ \\
\midrule
CoNLL & NER  & 0.356 & 1.00 & 0.0 pp \\
FNERD & NER  & 0.540 & 1.00 & 0.0 pp \\
WNUT  & NER  & 0.480 & 1.00 & 0.0 pp \\
TQA   & QA   & 0.441 & 1.00 & 0.0 pp \\
NQ    & QA   & 0.804 & 1.00 & 0.0 pp \\
STEM  & CLS  & 0.516 & 1.00 & 0.0 pp \\
Hum.  & CLS  & 0.333 & 1.00 & 0.0 pp \\
JSON  & JSON & 0.033 & 0.13$^\ast$ & 0.0 pp$^\dagger$ \\
\bottomrule
\end{tabular}
\caption{Audited $\hat{M} = \operatorname{ess\,sup} R$ per dataset (averaged over the five sampling-based models). ``floor gap'' $= (\mu-\alpha)/(M-\alpha) - (\mu-\alpha)/(1-\alpha)$ at $\alpha = 0.10$. $^\ast$JSON field-$F_1$ audits to $\hat{M}$ between $0.00$ and $0.33$ across models. $^\dagger$For JSON $\mu < \alpha$, so both floors are $0$. On all NER/QA/CLS datasets $\hat{M} = 1.00$, so the sharpened and conservative floors coincide.}
\label{tab:audit_M}
\end{table}

\subsection{Disjoint-Split Score Fusion}
\label{app:fusion_disjoint}

The main-text fusion numbers (\S\ref{sec:fusion}) are recomputed here under a protocol that respects Proposition~\ref{prop:fusion}'s independence requirement, replacing an earlier in-sample estimate. Using a three-way disjoint split (the seed-42 test set is byte-identical to the headline grid; the $60\%$ calibration block is halved into a fusion-training slice, on which the logistic-regression fusion weights are fit, and a CRC-calibration slice, on which thresholds are selected by the paper's own fixed-sequence scan), AUROC is read out-of-sample and individual-score baselines are certified on the identical CRC-calibration/test split. Over $136$ model$\times$dataset$\times\alpha$ rows ($34$ configurations averaged over $\alpha$):
\begin{itemize}\setlength{\itemsep}{.1em}
\item \textbf{AUROC.} Fusion improves on $23/34$ configurations ($68\%$; mean $+0.004$, median $+0.006$, max $+0.062$), with a worst-case loss of $-0.154$ (Qwen2.5-14B/MMLU-Humanities, where the small fusion-training slice overfits). This is roughly $7\times$ smaller than the in-sample $+0.026$ previously reported.
\item \textbf{Certification.} Fused vs.\ best-individual certified counts on the same split are $0$ vs.\ $7$ (Hoeffding), $0$ vs.\ $8$ (Bernstein), $2$ vs.\ $13$ (e-CRC). \emph{Fusion never certifies a configuration that no individual score already certifies} (0 fused-only certifications at every bound); in the two e-CRC cases where both certify, fusion lowers test abstention by $\approx 1.9$pp.
\end{itemize}
\textbf{Caveat.} The CRC-calibration slice here is half the headline calibration budget, since the released per-example files contain no separate training pool; both arms are equally starved, so this is a conservative test, not a proof that fusion cannot help at the full budget. It is, however, the disjointness-respecting estimate, and it shows fusion's benefit on this data is on ranking quality, not on certified-set size.

\subsection{Sensitivity to the Confidence Level}
\label{app:delta}

Re-running the seed-42 grid at $\delta \in \{0.05, 0.10, 0.20\}$ gives certified counts (Hoeffding/Bernstein/e-CRC) of $48/64/74$, $51/72/80$ (paper), and $58/79/89$. Two findings are worth stating. First, at $\delta = 0.20$, e-CRC produces $3$ violations of $89$ certifications ($3.4\%$, well within the $20\%$ budget: Qwen2.5-3B/JSON/\{token-margin, NLL\}/$\alpha{=}0.05$ and Qwen2.5-72B/CoNLL/token-margin/$\alpha{=}0.20$), positive evidence that the guarantee is not vacuously conservative. Second, nesting $\Lambda^*_{\mathrm{Hoeff}} \subseteq \Lambda^*_{\mathrm{Bern}}$ holds without exception at $\delta = 0.10$ and $0.20$, and fails in exactly $1$ of $2{,}148$ method-pairs at $\delta = 0.05$ (Ministral-8B/MMLU-Humanities/self-consistency/$\alpha{=}0.20$), consistent with Proposition~\ref{prop:ordering}'s $\delta$-dependent finite-sample threshold ($L = \log(2/\delta)$). The paper's ``nested without exception at $\alpha \leq 0.20$'' claim is thus a $\delta = 0.10$ statement.

\subsection{Per-Task Certified Counts and Multi-Seed Reversal}
\label{app:pertask}

Table~\ref{tab:pertask} breaks the seed-42 strict-target counts down by task (cells sum to the headline $51/72/80$). Essentially all strict-target certification comes from JSON alone ($47$ of $51$ Hoeffding certifications); NER/QA/CLS certify near-zero at $\alpha \leq 0.10$, exactly as the floor predicts.

\begin{table}[t]
\centering
\small
\setlength{\tabcolsep}{4pt}
\begin{tabular}{lcccc}
\toprule
Task ($n$/$\alpha$) & $\alpha{=}0.05$ & $\alpha{=}0.10$ & $\alpha{=}0.15$ & $\alpha{=}0.20$ \\
\midrule
NER ($81$)  & 0/0/0 & 0/0/0 & 0/2/2 & 2/2/2 \\
QA  ($40$)  & 0/0/0 & 0/0/0 & 0/1/1 & 1/1/1 \\
CLS ($36$)  & 0/0/0 & 0/0/0 & 0/0/1 & 1/1/1 \\
JSON ($22$) & 0/9/12 & 9/17/19 & 18/19/20 & 20/20/21 \\
\bottomrule
\end{tabular}
\caption{Per-task certified counts at seed 42 (cells are Hoeffding/Bernstein/e-CRC). Column-summed totals are $51/72/80$.}
\label{tab:pertask}
\end{table}

The relaxed-target Hoeffding--Bernstein reversal is seed-stable across seeds 42--46. At $\alpha = 0.30$, Hoeffding $>$ Bernstein at all five seeds (Hoeffding $42/42/42/41/43$ vs.\ Bernstein $35/38/38/36/39$); at $\alpha = 0.40$, all five ($57/56/59/57/57$ vs.\ $48/50/49/50/50$); at $\alpha = 0.25$, nearest the reversal onset, four of five seeds show Hoeffding $>$ Bernstein and one seed ties, consistent with the frontier's $n$-dependence (Proposition~\ref{prop:phase}). The seed-42 relaxed counts reported in the main text ($42$, $57$ Hoeffding) sit squarely inside these ranges.

\paragraph{Numerical demonstration.} Code \texttt{theory/saver.py}, \texttt{run\_saver.py}; results \texttt{saver\_results.json}. On the identical cross-dataset transfers of Table~\ref{tab:aci_xfer} (best score, static-CRC warm start, $\delta = 0.1$, $\rho = 25$), at $\alpha = 0.10$ (16 model$\times$transfer runs): static CRC 14/16 violations, ACI ($\gamma = 0.01$) 14/16, SAVeR-FF (full feedback) \textbf{0/16} at mean abstention $1.00$, since every transfer is infeasible (floors $0.33$--$0.83$ by Proposition~\ref{prop:impossibility}); the emit-only betting monitor reproduces Proposition~\ref{prop:s-e3}, paying an exploration cost and violating 14/16. Non-vacuity on feasible streams: qwen2.5-14b/JSON at $\alpha = 0.10$ emits 388/600 (65\%) at risk $0.000$ (0 violations); qwen2.5-14b/CoNLL at $\alpha = 0.40$ emits 3338/3453 (97\%) at risk $0.258$; qwen2.5-72b/CoNLL at $\alpha = 0.40$ emits 3316/3453 (96\%) at risk $0.242$; ministral-8b/CoNLL$\to$WNUT at $\alpha = 0.40$ (floor $0.06$) emits 718 (abstention $0.44$) at 0 violations. Abstention $\geq$ floor on all runs. A synthetic coverage check with known drifting/adversarial $q_t$ (2000 reps, $T = 1500$) gives empirical miscoverage $\leq 0.011 \ll \delta$ (drift-up $0.003$, sinusoid $0.009$, adversarial-burst $0.000$, iid-mid $0.011$), confirming Theorem~\ref{thm:s-e1}. \emph{Limitation.} The rigorous guarantee uses the Hoeffding mixture ($\sigma^2 = 1/4$) and over-abstains on low-risk streams; a variance-adaptive uniform boundary is a drop-in tightening we do not re-derive here to avoid an invalid guarantee.

\end{document}